\documentclass{article}

\usepackage{ifpdf}
\usepackage{fullpage}

\usepackage[utf8]{inputenc} 
\usepackage[T1]{fontenc}    

\usepackage[square,sort,compress]{natbib}

\usepackage{url}            
\usepackage{booktabs}       
\usepackage{nicefrac}       
\usepackage{microtype}      
\usepackage{graphicx}
\usepackage{caption}
\usepackage{subcaption}
\usepackage{amsfonts,mathtools}
\usepackage{algorithmic}
\usepackage[]{algorithm}
\usepackage{array}
\usepackage{bm}
\usepackage{mathrsfs}
\usepackage{xcolor}
\usepackage{float}
\usepackage{epstopdf}
\usepackage{amssymb}
\usepackage{enumitem}
\usepackage{amsmath}
\usepackage{multirow}
\usepackage{hyperref}
\hypersetup{colorlinks=true,citecolor=blue,urlcolor=black,linkcolor=black,bookmarks=false,hypertexnames=false}

\usepackage[capitalize,noabbrev]{cleveref}

\newcommand{\indep}{\perp \!\!\! \perp}

\usepackage{amsthm} 

\renewcommand{\algorithmicrequire}{ \textbf{Input:}} 
\renewcommand{\algorithmicensure}{ \textbf{Output:}}

\def \E {{\mathcal{E}}}

\makeatletter
\newcommand*\rel@kern[1]{\kern#1\dimexpr\macc@kerna}

\newcommand*\widebar[1]{%
  \kern 0.18em
  \hbox{%
    \vbox{%
      \hrule height 0.5pt 
      \kern0.35ex
      \hbox{%
        \kern-0.18em
        \ensuremath{#1}%
        \kern-0.18em
      }%
    }%
  }%
  \kern 0.18em
}

\theoremstyle{plain}
\newtheorem{theorem}{Theorem}[section]
\newtheorem{proposition}[theorem]{Proposition}
\newtheorem{lemma}[theorem]{Lemma}
\newtheorem{corollary}[theorem]{Corollary}
\theoremstyle{definition}

\newtheorem{assumption}[theorem]{Assumption}
\theoremstyle{remark}
\newtheorem{remark}[theorem]{Remark}

\newcommand{\trace}{\mathop{\mathrm{tr}}}
\newcommand{\norm}[1]{\ensuremath{\left\|#1\right\|}}	
\providecommand{\tr}[1]{\text{tr}\left(#1\right)}
\providecommand{\norm}[1]{\left \| #1 \right \|}
\providecommand{\vect}[1]{\mathrm{vec}\left(#1\right)}

\newcommand{\normI}[1]{{\left\vert\kern-0.25ex\left\vert\kern-0.25ex\left\vert #1 
    \right\vert\kern-0.25ex\right\vert\kern-0.25ex\right\vert}}

\begin{document}

\title{Adaptive Estimation of Graphical Models under Total Positivity}

	\author{ Jiaxi Ying 
        \thanks{\noindent Department of Mathematics, The Hong Kong University
of Science and Technology, Clear Water Bay, Hong Kong SAR, and HKUST Shenzhen-Hong Kong Collaborative Innovation Research Institute, Shenzhen, China; e-mail:
\href{mailto:jx.ying@connect.ust.hk}{\texttt{jx.ying@connect.ust.hk}}.}
	\and
	Jos\'{e} Vin\'{i}cius de M. Cardoso
	\thanks{\noindent Department of Electronic and Computer Engineering, The Hong Kong University of Science and Technology, Clear Water Bay, Hong Kong SAR; e-mail: \href{mailto:jvdmc@connect.ust.hk}{\texttt{jvdmc@connect.ust.hk}}.}
   \and
	Daniel P. Palomar
	\thanks{\noindent Department of Electronic and Computer Engineering, Department of Industrial Engineering and Data Analytics, The Hong Kong University of Science and Technology,
 Clear Water Bay, Hong Kong SAR; e-mail: \href{mailto:palomar@ust.hk}{\texttt{palomar@ust.hk}}.}
	}
\date{}

\maketitle

\begin{abstract}
We consider the problem of estimating (diagonally dominant) \textit{M}-matrices as precision matrices in Gaussian graphical models. These models exhibit intriguing properties, such as the existence of the maximum likelihood estimator with merely two observations for \textit{M}-matrices \citep{lauritzen2019maximum,slawski2015estimation} and even one observation for diagonally dominant \textit{M}-matrices \citep{truell2021maximum}. We propose an adaptive multiple-stage estimation method that refines the estimate by solving a weighted $\ell_1$-regularized problem at each stage. Furthermore, we develop a unified framework based on the gradient projection method to solve the regularized problem, incorporating distinct projections to handle the constraints of \textit{M}-matrices and diagonally dominant \textit{M}-matrices. A theoretical analysis of the estimation error is provided. Our method outperforms state-of-the-art methods in precision matrix estimation and graph edge identification, as evidenced by synthetic and financial time-series data sets.
\end{abstract}

\section{Introduction}\label{sec-1}
Total positivity, as a strong form of positive dependence, has found its applications in various fields such as factor analysis in psychometrics \citep{lauritzen2019maximum}, taxonomic reasoning \cite{lake2010discovering}, graph signal processing \citep{egilmez2017graph}, and financial markets \citep{Agrawal20}. For instance, stock markets exemplify total positivity, as stocks generally exhibit positive dependence due to the influence of market factors \citep{Agrawal20}. A distribution that satisfies total positivity is also known as multivariate totally positive of order two (MTP\textsubscript{2}). For a more in-depth understanding, we refer the reader to \citet{fallat2017total}. In the Gaussian context, the concept of MTP\textsubscript{2} becomes more straightforward: a Gaussian distribution is MTP\textsubscript{2} if and only if the precision matrix (\textit{i.e.}, inverse covariance matrix) is an \textit{M}-matrix \citep{karlin1983m}. In this paper, our focus lies on estimating (diagonally dominant) \textit{M}-matrices as precision matrices in Gaussian graphical models.

Estimation of precision matrices in general Gaussian graphical models have been extensively studied in the literature. A well-known method, known as graphical lasso \citep{banerjee2008model,d2008first}, is formulated as an $\ell_1$-regularized Gaussian maximum likelihood estimation. Numerous extensions of graphical lasso have been explored \citep{friedman2008sparse,ravikumar2011high,lam2009sparsistency,fan2014strong,loh2015regularized,loh2017support,sun2018graphical,hsieh2014quic}.  Other notable, though not exhaustive, works include graphical Dantzig \citep{yuan2010high}, CLIME \citep{cai2011constrained}, and TIGER \citep{liu2017tiger}. Recent research has illuminated intriguing properties of incorporating additional \textit{M}-matrix constraint, proving advantageous in the realms of high-dimensional statistics and signal processing on graphs.

Recent studies \citep{lauritzen2019maximum,slawski2015estimation,lauritzen2021total} show that the \textit{M}-matrix constraint significantly reduces the sample size required for the maximum likelihood estimator (MLE) to exist. Specifically, \citet{lauritzen2019maximum,slawski2015estimation} established that the MLE under the \textit{M}-matrix constraint exists if the sample size $n \geq 2$, regardless of the underlying dimension $p$, substantially reducing the $n \geq p$ requirement for general Gaussian distributions. \citet{soloff2020covariance} revealed that the \textit{M}-matrix constraint, serving as implicit regularization, is vital for achieving the minimax rate of $\sqrt{\frac{\log p}{n}}$; without this constraint, the rate cannot exceed $\sqrt{\frac{p}{n}}$. Remarkably, \citet{lauritzen2021total} proved that under the MTP\textsubscript{2} constraint, the MLE in binary distributions may exist with only $n = p$ observations, in contrast to the $2^p$ observations required in the absence of this constraint.

The diagonally dominant \textit{M}-matrix constraint is increasingly drawing interest in Gaussian models. In these distributions, diagonally dominant \textit{M}-matrices, as precision matrices, exhibit a property called log-$L^\natural$-concave (LLC) \citep{robeva2021maximum} or strong MTP\textsubscript{2} \citep{rottger2021total}. This property is essential for analyzing positive dependence in multivariate Pareto distributions. A recent study \citep{truell2021maximum} revealed that Gaussian models under LLC/strong MTP\textsubscript{2} act as a convex relaxation for Brownian motion tree models—a class of Gaussian models on trees—and showed that its MLE exists almost surely when $n = 1$. This finding was established by connecting Laplacian constrained Gaussian Markov random fields (LGMRF) \citep{ying2020nonconvex,cardoso2021graphical,cardoso2022learning,kumar2019unified} and leveraging results on MLE existence in LGMRF with a single observation \citep{ying2021minimax}.

The estimation of MTP\textsubscript{2} graphical models has become a growing area of interest in the field of signal processing on graphs \citep{shuman2013emerging,ortega2018graph,dong2019learning}, which focuses on handling data defined on irregular graph domains. Precision matrices in MTP\textsubscript{2} graphical models belong to the class of generalized Laplacian matrices \citep{biyikoglu2007laplacian}, each of which is a symmetric matrix with non-positive off-diagonal entries that are associated with a graph. The eigenvectors of a generalized Laplacian define a graph Fourier transform \citep{shuman2013emerging}, supported by the discrete nodal domain theorem \citep{BRIANDAVIES200151} that an eigenvector corresponding to a larger eigenvalue is associated with a higher frequency. We note that such a spectral property does not hold for general positive definite matrices.

One approach to estimating MTP\textsubscript{2} graphical models is MLE \citep{lauritzen2019maximum,slawski2015estimation}, which implicitly promotes sparsity using the \textit{M}-matrix constraint. However, it cannot adjust sparsity levels to specific values, a feature often desired in practice. The $\ell_1$-norm regularized MLE \citep{egilmez2017graph,pavez2016generalized,cai2021fast,deng2020fast} offers improved sparsity control and can be solved using techniques like block coordinate descent \citep{egilmez2017graph, pavez2016generalized}, proximal point algorithm \citep{deng2020fast}, and projected Newton-like methods \citep{cai2021fast}. Notably, \citet{pavez2018learning} found that graph components and some zero patterns can be determined through thresholded sample covariance matrices. Estimating diagonally dominant \textit{M}-matrices as precision matrices is explored in \citet{egilmez2017graph}. Additionally, \citet{wang2020learning} proposed a graph structure learning algorithm based on conditional independence testing, eliminating the need for adjustment of tuning parameters.

This paper focuses on estimating (diagonally dominant) \textit{M}-matrices as precision matrices in MTP\textsubscript{2} Gaussian graphical models. The main contributions of this paper are threefold:

\begin{itemize}
\itemsep0em 
\item We propose an adaptive multiple-stage estimation method that refines the estimate by solving a weighted $\ell_1$-regularized problem at each stage. Then we develop a unified framework based on the gradient projection method to solve the regularized problem, equipped with distinct projections to handle the constraints of \textit{M}-matrices and diagonally dominant \textit{M}-matrices.

\item We provide a thorough theoretical analysis of the estimation error for our method, which comprises two components: optimization error and statistical error. The optimization error decays at a linear rate, highlighting the progressive refinement of the estimate across iterative stages, whereas the statistical error captures the intrinsic statistical rate.

\item Experiments on synthetic and financial time-series data demonstrate that our method outperforms state-of-the-art methods, concurrently achieving lower precision matrix estimation error and higher graph edge selection accuracy. Additionally, we observe that incorporating \textit{M}-matrix constraint enhances the robustness regarding the selection of the regularization parameter.
\end{itemize}

Lower and upper case bold letters denote vectors and matrices, respectively. $[p]$ denotes the set $\{1, \ldots, p\}$. $\norm{\bm x}_{\max} = \max_{i} |x_i|$. $\mathbb{S}_{++}^p$ denotes the set of $p \times p$ positive definite matrices. $\lambda_{\max}( \bm X)$ and $\lambda_{\min}( \bm X)$ denote the maximum and minimum eigenvalues of $\bm X$, respectively. For matrix $\bm X$ and set $S$, $\bm X_S$ denotes a vector with dimension $\left| S\right|$, containing the entries of $\bm X$ indexed by $S$.

\section{Preliminaries and Problem Formulation}\label{sec-2}

We first introduce preliminaries about MTP\textsubscript{2} Gaussian graphical models, then present the problem formulation.

\subsection{MTP\textsubscript{2} Gaussian Graphical Models}\label{sec-Preliminaries}

Let $\mathcal{G}=(\mathcal{V},\mathcal{E})$ be an undirected graph with the set of nodes $\mathcal{V}$ and the set of edges $\mathcal{E}$. Associating a random vector $\bm x \sim \mathcal{N}(\bm 0, \bm \Sigma^\star)$ with graph $\mathcal{G}$ forms a Gaussian graphical model that satisfies the following properties:
\begin{equation*}
\begin{aligned}
	&\Theta^\star_{ij} \neq 0	\iff \{i,j\} \in \E \;\forall \; i\neq j, \\
	&\Theta^\star_{ij}=0 \iff x_i \indep x_j \; \vert \;  \bm x_{ [p]\backslash\{ i, j \}},
	\end{aligned}
\end{equation*}
where $x_i \indep x_j \; \vert \;  \bm x_{ [p]\backslash\{ i, j \}}$ indicates that $x_i$ is conditionally independent of $x_j$ given the other random variables. Consequently, graph $\mathcal{G}$ characterizes the sparsity pattern of the precision matrix $\bm \Theta^\star$, where missing edges represent conditional independence between random variables. Estimating the precision matrix is a crucial task in Gaussian graphical models.

We concentrate on estimating precision matrices for random variables following an MTP\textsubscript{2} Gaussian distribution. A multivariate Gaussian distribution with a positive definite covariance matrix $\bm \Sigma^\star$ is MTP\textsubscript{2} if and only if its precision matrix $\bm \Theta^\star$ is a symmetric \textit{M}-matrix \citep{karlin1983m}, \textit{i.e.}, $\Theta^\star_{ij} \leq 0$ for all $i \neq j$. This is equivalent to stating that all partial correlations are nonnegative, where the partial correlation between any two variables $x_i$ and $x_j$ conditioned on all other variables equals $- \Theta^\star_{ij} / \sqrt{ \Theta^\star_{ii} \Theta^\star_{jj}}$.

\subsection{Problem Formulation}\label{sec-PR}
 
Consider a zero-mean random vector $\bm x = \left( x_1, \ldots, x_p \right)$ following a Gaussian distribution $\bm x \sim \mathcal{N}(\bm 0, \bm \Sigma^\star)$. Our goal is to estimate the precision matrix $\bm \Theta^\star := (\bm \Sigma^\star)^{-1}$, which is a (diagonally dominant) \textit{M}-matrix, given $n$ independent and identically distributed observations $\bm x^{(1)}, \ldots, \bm x^{(n)} \in \mathbb{R}^p$. 
 
To estimate the precision matrix under MTP\textsubscript{2} Gaussian distributions, we can solve the penalized Gaussian maximum likelihood estimation problem as follows:
\begin{equation}\label{MLE}
\underset{\bm \Theta \in \mathcal{S}}{\mathsf{minimize}} \,  - \log \det (\bm \Theta) + \trace \big(\bm \Theta \widehat{\bm \Sigma}\big) + \sum_{i \neq j} h_{\lambda} \left( \left| \Theta_{ij} \right| \right),
\end{equation}
where $h_{\lambda}$ is a sparsity penalty function, such as the $\ell_1$-norm considered in \citet{egilmez2017graph}, $\widehat{\bm \Sigma} = \frac{1}{n} \sum_{i=1}^n \bm x^{(i)} (\bm x^{(i)})^\top$ is the sample covariance matrix, and $\mathcal{S}$ is a feasible set. We consider two cases for $\mathcal{S}$:
\begin{equation}\label{set-M}
\mathcal{M}^p = \left\lbrace  \bm \Theta \in \mathbb{S}_{++}^p \, | \, \Theta_{ij} = \Theta_{ji} \leq 0, \forall i \neq j \right\rbrace,
\end{equation}
and
\begin{equation}\label{set-MD}
\mathcal{M}_{D}^p   =  \left\lbrace  \bm \Theta \in \mathbb{S}_{++}^p | \Theta_{ij} = \Theta_{ji} \leq 0, \forall i \neq j, \bm \Theta \bm 1 \geq \bm 0  \right\rbrace, 
\end{equation}
where $\mathcal{M}^p$ represents the set of all symmetric positive definite \textit{M}-matrices, and $\mathcal{M}_{D}^p$ adds constraint $\bm \Theta \bm 1 \geq \bm 0$ compared to $\mathcal{M}^p$, leading $\bm \Theta$ to be a diagonally dominant \textit{M}-matrix. We propose an adaptive method for estimating \textit{M}-matrices and diagonally dominant \textit{M}-matrices.

\section{Proposed Method}\label{sec-3}
In this section, we first propose an adaptive multiple-stage estimation method, then develop a unified framework based on projected gradient descent for solving each stage. We incorporate distinct projections for estimating \textit{M}-matrices and diagonally dominant \textit{M}-matrices as precision matrices.

\subsection{Adaptive Estimation}\label{sec-PE}

 Our method comprises multiple stages, where each stage refines the previous stage's estimate by solving a weighted $\ell_1$-regularized problem. Specifically, at the $k$-th stage, we obtain the estimate $\widehat{\bm \Theta}^{(k)}$ by solving the following problem:
\begin{equation}\label{sec-stage}
\underset{\bm \Theta \in \mathcal{S}}{\mathsf{minimize}} \,  - \log \det (\bm \Theta) + \trace \big(\bm \Theta \widehat{\bm \Sigma} \big) + \sum_{i \neq j}  \lambda_{ij}^{(k)}  \left| \Theta_{ij} \right|,
\end{equation} 
where each $\lambda_{ij}^{(k)} = p_\lambda ( | \widehat{\Theta}^{(k-1)}_{ij} |)$ with $p_\lambda$ the weight-updating function, and $\widehat{\bm \Theta}^{(k-1)}$ is the estimate obtained at the $(k-1)$-th stage.

When the estimate $\widehat{\bm \Theta}^{(k-1)}$ from the previous stage exhibits a large coefficient for its $(i,j)$ entry, it is reasonable to assign a small $\lambda_{ij}^{(k)}$ in Problem~\eqref{sec-stage}. Therefore, $p_\lambda$ should be monotonically non-increasing. In the initial stage, we set each $\lambda_{ij}^{(1)} = \lambda$, reducing Problem~\eqref{sec-stage} to the well-studied $\ell_1$-regularized estimation \citep{egilmez2017graph,ying2021fast}. Our method refines the $\ell_1$-norm estimate in subsequent stages, offering an improvement over the standard $\ell_1$-norm approach.

By choosing $p_\lambda$ as the derivative of the sparsity penalty function $h_\lambda$ in Problem~\eqref{MLE}, the sequence $\big\lbrace \widehat{\bm \Theta}^{(k)} \big\rbrace_{k\geq 1}$ converges to a stationary point of Problem~\eqref{MLE} with a nonconvex penalty. In this context, our method aligns with the local linear approximation method \citep{zou2008one} and the broader framework of multi-stage convex relaxation \citep{zhang2010analysis}. Nonconvex approaches may encounter issues with sensitivity to the regularization parameter selection, as illustrated in Figure~\ref{fig:sensitivity-line}. However, we demonstrate that incorporating the \textit{M}-matrix constraint considerably enhances the robustness regarding the selection of regularization parameter.

\subsection{Gradient Projection Method}\label{sec:sec3.2}

For each $\bm \Theta \in \mathcal{S}$, where $\mathcal{S} = \mathcal{M}^p$ or $\mathcal{M}_D^p$, all off-diagonal entries are nonpositive, leading to the following representation of Problem \eqref{sec-stage}:
\begin{equation}\label{eq:problem-S}
\underset{\bm \Theta \in \mathcal{S}}{\mathsf{minimize}} \,  - \log \det (\bm \Theta) + \trace \big(\bm \Theta \widehat{\bm \Sigma} \big) - \sum_{i \neq j}  \lambda_{ij}^{(k)} \Theta_{ij}.
\end{equation}
We design a gradient projection method to solve Problem~\eqref{eq:problem-S}. Initially, we perform a gradient descent step, $\bm \Theta_{t} - \eta_t \nabla f ( {\bm \Theta}_{t} )$, with $\nabla f$ representing the gradient of the objective function $f$. As this step may exceed the constraint set $\mathcal{S}$, we subsequently project it back onto $\mathcal{S}$. However, such projection onto $\mathcal{S}$ does not exist due to its nonclosedness, resulting from the nonclosedness of $\mathbb{S}_{++}^p$. Instead, we express the set $\mathcal{S}$ as the intersection of a closed set $\mathcal{S}_d$ and the open set $\mathbb{S}_{++}^p$:
\begin{equation}\label{eq:Sd}
\mathcal{S} = \mathcal{S}_d \cap \mathbb{S}_{++}^p.
\end{equation} 
We handle the constraints of $\mathcal{S}_d$ and $\mathbb{S}_{++}^p$ separately. The closed set $\mathcal{S}_d$ constraint is manageable through projection, leading to the following projected gradient descent:
\begin{equation}\label{pgd}
\bm \Theta_{t+1}  = \mathcal{P}_{\mathcal{S}_d} \big( \bm \Theta_{t} - \eta_t \nabla f ( {\bm \Theta}_{t} ) \big),
\end{equation}
where $\mathcal{P}_{\mathcal{S}_d}$ denotes the projection onto the set $\mathcal{S}_d$ with respect to the Frobenius norm.

The constraint of $\mathbb{S}_{++}$ cannot be managed by projection due to its nonclosedness. One efficient approach to enforce positive definiteness is using a line search procedure \citep{hsieh2014quic}. Specifically, at the $t$-th iteration, we try the step size $\eta_t \in \sigma\left\lbrace \beta^0, \beta^1, \beta^2, \ldots \right\rbrace$, with $\beta \in (0, 1)$ and $\sigma>0$, until we find the smallest $m \in \mathbb{N}$ such that the iterate $\bm \Theta_{t+1}$ in \eqref{pgd}, with $\eta_t = \sigma\beta^m$, ensures $\bm \Theta_{t+1} \in \mathbb{S}_{++}^p$ and satisfies:
\begin{equation}\label{eq:line-search}
f ( \bm \Theta_{t+1} ) \leq f ( \bm \Theta_{t} )  - \alpha \eta_t \| G_{\frac{1}{\eta_t}} (\bm \Theta_t) \|^2_{\mathrm{F}},
\end{equation}
where $\alpha \in (0,1)$, and $G_{\frac{1}{\eta_t}} (\bm \Theta_t)= \frac{1}{\eta_t} ( \bm \Theta_t - \bm \Theta_{t+1} )$. We present our algorithm in Algorithm~\ref{algo-1}, where $\bm \Lambda^{(k)}$ contains all weights $\lambda_{ij}^{(k)}$, and diagonal entries are set to zero.

\begin{theorem}\label{converge-pgd}
The sequence $\big\lbrace \bm \Theta_{t} \big\rbrace_{t \geq 0}$ established by Algorithm~\ref{algo-1} converges to the optimal solution of Problem~\eqref{eq:problem-S}.
\end{theorem}
The proof is provided in Appendix~\ref{sec:proof-pgd}. It is important to note that the existence of the minimizer for Problem~\eqref{eq:problem-S} is assumed throughout this paper, which can be guaranteed almost surely if the sample size $n \geq 2$ for the case of $\mathcal{S} = \mathcal{M}^p$ \citep{lauritzen2019maximum,slawski2015estimation} and $n \geq 1$ for the case of $\mathcal{S} = \mathcal{M}^p_D$ \citep{truell2021maximum}.

Although our method needs to solve a sequence of log-determinant programs \eqref{eq:problem-S}, numerical results indicate a rapid decrease in the number of iterations required for solving \eqref{eq:problem-S} as $k$ increases (see Figure~\ref{fig:Convergence_illustration}). This accelerated convergence can be attributed to the use of the estimate from the previous stage as an initial point, thereby providing a warm start for faster convergence.

We now present notable extensions of the proposed gradient projection algorithm. First, the algorithm can be readily extended to solve Problem~\eqref{MLE} with a nonconvex penalty. Second, the algorithm can be further adapted to solve the following problem for the case of multivariate \textit{t}-distribution: 
\begin{equation*}
\underset{\bm \Theta \in \mathcal{S}}{\mathsf{minimize}}   - \log \det (\bm \Theta) + \frac{p + \nu}{n}   \sum_{i=1}^n \log  \big( 1 + \frac{1}{{\nu}}(\bm x^{(i)})^\top    \bm \Theta \bm x^{(i)} \big),
\end{equation*}
where $\nu$ denotes the number of degrees of freedom, and each $\bm x^{(i)}$ is an observation. This formulation can be employed to learn positive partial correlation graphs. Notably, elliptical MTP\textsubscript{2} distributions are highly restrictive, and the total positivity property is not applicable to \textit{t}-distributions \citep{rossell2021dependence}. It is worth mentioning that, unlike block coordinate descent \citep{egilmez2017graph} and projected Newton-like methods \citep{cai2021fast}, the proposed algorithm can handle these extensions, resulting in a more flexible and versatile approach.

\begin{algorithm}[hb]
   \caption{Solve Problem~\eqref{eq:problem-S}}\label{algo-1} 
\begin{algorithmic}[1] 
   \STATE {\bfseries Input:} Sample covariance matrix $\widehat{\bm \Sigma}, \bm \Lambda^{(k)}, \sigma, \alpha, \beta$.
    \vspace{0.1cm}  
    \FOR{$t = 0, 1, 2, \ldots$} 
    \STATE Compute $\nabla f (\bm \Theta_t) = - \bm \Theta_t^{-1} + \widehat{\bm \Sigma} - \bm \Lambda^{(k)}$;
    \STATE $m \leftarrow 0$;
    \REPEAT
    \STATE Update $\bm \Theta_{t + 1} =  \mathcal{P}_{\mathcal{S}_d} \left(\bm \Theta_t - \sigma \beta^m \nabla f(\bm \Theta_t) \right)$;
    \STATE $m\leftarrow m+1$;
    \UNTIL{$\bm \Theta_{t+1} \in \mathbb{S}_{++}^p$ and}
    \STATE {$ \qquad \ f ( \bm \Theta_{t+1} ) \leq f ( \bm \Theta_{t} )  - \alpha \sigma \beta^m \| G_{\frac{1}{\eta_t}} (\bm \Theta_t) \|^2_{\mathrm{F}}$};   
    \ENDFOR   
	\STATE {\bfseries Output:} $\widehat{\bm \Theta}^{(k)}$.
\end{algorithmic}
\end{algorithm}

\subsection{Computation of Projections}\label{sec:sec3.3}

We present the computation of projection $\mathcal{P}_{\mathcal{S}_d}$ in \eqref{pgd} for both cases of estimating \textit{M}-matrices and diagonally dominant \textit{M}-matrices, \textit{i.e.}, $\mathcal{S} = \mathcal{M}^p$ and $\mathcal{M}^p_D$ in Problem~\eqref{eq:problem-S}.

\subsubsection{Estimation of M-matrices}\label{sec:sec3.21}
For the case of estimating an \textit{M}-matrix as the precision matrix, the constraint set of Problem~\eqref{eq:problem-S} is set as $\mathcal{S} = \mathcal{M}^p$, defined in \eqref{set-M}. The closed set $\mathcal{S}_d$ in \eqref{eq:Sd} then becomes
\begin{equation}
\mathcal{S}_d = \left\lbrace  \bm \Theta \in \mathbb{R}^{p \times p} \, | \, \Theta_{ij} = \Theta_{ji} \leq 0, \ \forall \, i \neq j \right\rbrace,
\end{equation}
which can be written as the intersection of two sets:
\begin{equation}
\mathcal{S}_d = \mathcal{S}_A \cap \mathcal{S}_B,
\end{equation}
where $\mathcal{S}_A:= \left\lbrace  \bm \Theta \in \mathbb{R}^{p \times p} \, | \, \Theta_{ij} = \Theta_{ji}, \forall i \neq j \right\rbrace$ and $\mathcal{S}_B:= \left\lbrace  \bm \Theta \in \mathbb{R}^{p \times p} \, | \, \Theta_{ij} \leq 0, \forall i \neq j \right\rbrace$. If $\bm \Theta_{t}$ is symmetric, then the update $\mathcal{P}_{\mathcal{S}_B} \big( \bm \Theta_{t} - \eta_t \nabla f ( {\bm \Theta}_{t} ) \big)$ preserves symmetry. Thus, we only need to project $ \bm \Theta_{t} - \eta_t \nabla f ( {\bm \Theta}_{t} )$ onto $\mathcal{S}_B$. As a result, the iterate \eqref{pgd} can be simplified to
\begin{equation}\label{pgd-new-new}
\bm \Theta_{t+1}  = \mathcal{P}_{\mathcal{S}_B} \big( \bm \Theta_{t} - \eta_t \nabla f ( {\bm \Theta}_{t} ) \big),
\end{equation}
where the projection $\mathcal{P}_{\mathcal{S}_B}$ can be computed as follows:
\begin{equation} \label{prj-sb}
\left[ \mathcal{P}_{\mathcal{S}_B} (\bm X) \right]_{ij} =\left\{
\begin{array}{lcl}
\min \left ( X_{ij}, 0 \right)  &  & \mathrm{if} \, i \neq j,\\
 X_{ij}  & & \mathrm{if}  \, i = j.
\end{array}
\right.
\end{equation}
By initializing with a symmetric $\bm \Theta_0$, every point in the sequence $\{\bm \Theta_t \}_{t\geq 0}$ maintains symmetry.

\subsubsection{Estimation of Diagonally Dominant M-matrices}\label{sec:sec3.22}
For the case of estimating a diagonally dominant \textit{M}-matrix as the precision matrix, we set $\mathcal{S} = \mathcal{M}_D^p$ in \eqref{eq:problem-S}, with $\mathcal{M}_D^p$ defined in \eqref{set-MD}. The closed set $\mathcal{S}_d$ in \eqref{eq:Sd} then becomes
\begin{equation}\label{SD-MD}
 \mathcal{S}_d  =   \lbrace  \bm \Theta \in \mathbb{R}^{p\times p} \, | \, \Theta_{ij} = \Theta_{ji} \leq 0, \ \forall \, i \neq j, \ \bm \Theta \bm 1 \geq \bm 0 \rbrace, 
\end{equation}
which can be written as the intersection of two sets:
\begin{equation}
\mathcal{S}_d = \mathcal{S}_A \cap \mathcal{S}_C,
\end{equation}
where $\mathcal{S}_C:= \left\lbrace  \bm \Theta \in \mathbb{R}^{p \times p} \, | \, \bm \Theta \bm 1 \geq \bm 0, \Theta_{ij} \leq 0, \forall i \neq j \right\rbrace$. 
The iterate \eqref{pgd} is then written as:
\begin{equation}\label{pgd-dm}
\begin{gathered}
\bm \Theta_{t+1}  = \mathcal{P}_{\mathcal{S}_A \cap \mathcal{S}_C} \big( \bm \Theta_{t} - \eta_t \nabla f ( {\bm \Theta}_{t} ) \big).
\end{gathered}
\end{equation}
We note that, in contrast to \eqref{pgd-new-new}, the iterate \eqref{pgd-dm} cannot be simplified, as $\mathcal{P}_{\mathcal{S}_C} \big( \bm \Theta_{t} - \eta_t \nabla f ( {\bm \Theta}_{t} ) \big)$ does not guarantee symmetry.

We now present how to compute the projection $\mathcal{P}_{\mathcal{S}_A \cap \mathcal{S}_C}$ in \eqref{pgd-dm}, which can be written as the minimizer of the following problem:
\begin{equation}\label{eq:pro-sac}
\begin{gathered}
\mathcal{P}_{\mathcal{S}_A \cap \mathcal{S}_C} (\bm Y) := \underset{\bm X \in \mathcal{S}_A \cap \mathcal{S}_C}{\mathsf{arg \, min}} \, \norm{\bm X - \bm Y}_{\mathrm{F}}^2.
\end{gathered}
\end{equation}
Problem \eqref{eq:pro-sac} is a projection problem of a given point $\bm Y$ onto the intersection of the sets $\mathcal{S}_A$ and $\mathcal{S}_C$. To solve Problem~\eqref{eq:pro-sac}, we design an algorithm based on Dykstra's projection \citep{boyle1986method}, which aims to find the nearest point projection onto the intersection of closed convex sets. The algorithm is summarized in Algorithm \ref{algo-proj}, with $\bm Q_k$ denoting the increment associated with the set $\mathcal{S}_C$. We do not need to introduce the increment associated with $\mathcal{S}_A$ for convergence, since $\mathcal{S}_A$ is a subspace.
\begin{theorem}\label{Dykstra}
The sequences $\left\lbrace \bm A_k \right\rbrace$ and $\left\lbrace \bm C_k \right\rbrace$ generated by Algorithm \ref{algo-proj} converge to the minimizer of Problem~\eqref{eq:pro-sac}.
\end{theorem}
The convergence established in Theorem~\ref{Dykstra} is based on the convergence results of Dykstra's projection algorithm, as presented in \citet{boyle1986method}.

\begin{algorithm}
	\caption{Compute $\mathcal{P}_{\mathcal{S}_A \cap \mathcal{S}_C} (\bm Y)$} \label{algo-proj} 
		\begin{algorithmic}[1]
		\STATE {\bfseries Input:} $\bm C_0 =  \bm Y$, $\bm Q_0 = \bm 0$, $\epsilon$;
		\STATE { $k = 1$};
		\REPEAT 
		\STATE $\bm A_{k} = \mathcal{P}_{\mathcal{S}_A} \big(\bm C_{k-1} \big)$;
       \STATE $\bm C_{k} = \mathcal{P}_{\mathcal{S}_C} \big( \bm A_{k} + \bm Q_{k-1} \big)$;
       \STATE $\bm Q_{k} = \bm A_{k}  - \bm C_{k} + \bm Q_{k-1} $;
       \STATE {$k = k+1$};
		\UNTIL{$\norm{\bm A_k - \bm C_k}_{\mathrm{F}} < \epsilon$}
		\STATE {\bfseries Output:} $\mathcal{P}_{\mathcal{S}_A \cap \mathcal{S}_C} (\bm Y) = \bm A_k$. 
	\end{algorithmic}
\end{algorithm}

In what follows, we present the computation of $\mathcal{P}_{\mathcal{S}_A}$ and $\mathcal{P}_{\mathcal{S}_C}$ from Algorithm \ref{algo-proj}. The computation of $\mathcal{P}_{\mathcal{S}_A}$ is straightforward:
\begin{equation*}
\mathcal{P}_{\mathcal{S}_A} (\bm Y) = \big(\bm Y + \bm Y^\top \big)/2.
\end{equation*}
Projection $\mathcal{P}_{\mathcal{S}_C}$ is defined as follows:
\begin{equation}\label{eq:pro-sc}
\mathcal{P}_{ \mathcal{S}_C} (\bm Y) := \underset{\bm X \in \mathcal{S}_C}{\mathsf{arg \, min}} \, \frac{1}{2}\norm{\bm X - \bm Y}_{\mathrm{F}}^2.
\end{equation}
To compute $\mathcal{P}_{\mathcal{S}_C}$, we solve Problem \eqref{eq:pro-sc} row by row. For the $r$-th row, we address the following problem:
\begin{equation}\label{P_B_row}
\underset{\bm x \in \mathbb{R}^p}{\mathsf{minimize}}  \    \frac{1}{2}\norm{ \bm x - \bm y }^2,  \qquad\mathsf{subject~to} \ \bm x^\top \bm 1 \geq 0, \, \bm x_{\setminus r} \leq \bm 0,
\end{equation} 
where $\bm y \in \mathbb{R}^p$ contains all entries of the $r$-th row of $\bm Y$, and $\bm x_{\setminus r} \in \mathbb{R}^{p-1}$ includes all entries of $\bm x$ except the $r$-th one. 

Proposition \ref{solution-PB} below, proven in Appendix~\ref{appendix-pb}, provides the optimal solution of Problem \eqref{P_B_row}, which is a variant of the projection onto the simplex \citep{condat2016fast}.

\begin{proposition}\label{solution-PB}
Let $\hat{\bm x}$ be the optimal solution of Problem \eqref{P_B_row}, which can be obtained as follows:

\begin{itemize}
\item If $y_r \geq -\sum_{i \in [p] \setminus r} \min (y_i, \, 0)$, then $\hat{x}_r = y_r$, and $\hat{x}_i = \min (y_i, \, 0)$ for $i \neq r$.

\item If $y_r < -\sum_{i \in [p] \setminus r} \min (y_i, \, 0)$, then $\hat{x}_r = y_r + \rho$, and $\hat{x}_i = \min (y_i + \rho, \, 0)$ for $i \neq r$, where $\rho$ satisfies $\rho + \sum_{i \in [p] \setminus r} \min (y_i + \rho, \, 0) + y_r =0$.
\end{itemize}
\end{proposition}

\begin{remark}
Let $g(\rho) := \rho + \sum_{i \in [p] \setminus r} \min (y_i + \rho, \, 0) + y_r$. Note that there exists a unique solution to the piecewise linear equation $g(\rho) =0$. On one hand, $g(0) <0$, and $g(a) >0$ for any $a > \max_i |y_i|$. Therefore, there exits at least one solution to $g(\rho) =0$ since $g(\rho)$ is continuous. On the other hand, the solution is unique because $g(\rho)$ is a monotone function that is strictly increasing.
\end{remark}

Let $\bm y_{\setminus r} \in \mathbb{R}^{p-1}$ denote the vector that contains all entries of $\bm y$ except the $r$-th one. To solve the equation $g(\rho) =0$, we follow the procedures outlined in \citet{palomar2005practical}, which consists of three steps:
\begin{enumerate}
\item Sort $\bm y_{\setminus r}$ and obtain the sorted version $\tilde{\bm y}_{\setminus r}$, where $[\tilde{\bm y}_{\setminus r}]_1 \leq [\tilde{\bm y}_{\setminus r}]_2 \leq \cdots \leq [\tilde{\bm y}_{\setminus r}]_{p-1}$.

\item Find $M := \underset{1 \leq m \leq p-1}{\max} \Big\lbrace  m \big |  \frac{ y_r + \sum_{i=1}^{m} [\tilde{\bm y}_{\setminus r}]_i}{m + 1} > [\tilde{\bm y}_{\setminus r}]_m \Big\rbrace$.

\item Compute $\rho = \frac{- y_r - \sum_{m=1}^{M} [\tilde{\bm y}_{\setminus r}]_m}{M + 1}$.
\end{enumerate}
We note that sorting the vector $\bm y_{\setminus r}$ is the most computationally demanding step in the above procedures, typically necessitating $\mathcal{O}(p \log p)$ operations.

\section{Analysis of Estimation Error}\label{sec-theory}
In this section, we provide theoretical analysis of the estimation error for the proposed method.


Let $\bm \Sigma^\star$ and $\bm \Theta^\star$ denote the underlying covariance and precision matrices, respectively, where $\bm \Theta^\star \in \mathcal{S}$ ( $\mathcal{S} = \mathcal{M}^p$ or $\mathcal{M}_D^p$). Define $S^\star$ as the support set of the underlying precision matrix $\bm \Theta^\star$, excluding the diagonal entries, \textit{i.e.},
\begin{equation}\label{eq:set-S}
S^\star := \big\{ ( i, j ) \in [p]^2  \, | \, \Theta^{\star}_{ij} \neq 0, \, i \neq j \big\}.
\end{equation}

We require some mild conditions for the weight-updating function $p_{\lambda}$ in Assumption \ref{assumption 1} and the underlying precision matrix $\bm \Theta^\star$ in Assumption \ref{assumption 2}.

\begin{assumption} \label{assumption 1}
The function $p_{\lambda} : \mathbb{R}_+ \to \mathbb{R}_+ $ satisfies the following conditions:

\begin{enumerate}
\itemsep0em 
\item $p_{\lambda} (x)$ is monotonically nonincreasing, and $p_{\lambda} (0)= \lambda $;
\item There exists a $\gamma > 0$ such that $p_{\lambda} (x)= 0$ for $x \geq \gamma \lambda$, and $p_{\lambda} (c_0 \lambda) \geq \frac{\sqrt{2}}{2}\lambda$, where $c_0 = 12 \lambda^2_{\max}( \bm \Theta^\star)$ is a constant.
\end{enumerate}
\end{assumption}

\begin{assumption} \label{assumption 2}
The underlying precision matrix $\bm \Theta^\star$ has at most $s$ nonzero entries, and the off-diagonal nonzero entries satisfy $\min_{ (i, j) \in S^\star} | \Theta_{ij}^{\star} | \geq  (c_0 +\gamma) \lambda$, with $c_0$ and $\gamma$ defined in Assumption \ref{assumption 1}. There exists a constant $\delta$ such that $0 < \delta \leq \lambda_{\min} (\bm \Theta^\star) \leq \lambda_{\max} (\bm \Theta^\star) \leq 1/\delta < \infty$.
\end{assumption}

\begin{remark}
Assumption~\ref{assumption 1} fundamentally covers folded concave penalties, such as smooth clipped absolute deviation (SCAD) \citep{fan2001variable} and minimax concave penalty (MCP) \citep{zhang2010nearly}. In Assumption~\ref{assumption 2}, the conditions on the underlying precision matrix are relatively mild, as the regularization parameter $\lambda$ in our theorems takes the order of $\sqrt{ \log p /n}$, which could be small when the sample size $n$ increases. The assumption regarding minimal magnitude of signals is often utilized in nonconvex optimization analysis \citep{wang2016precision,ying2020does}.
\end{remark}

The following theorem, proven in Appendix~\ref{proof-sec}, provides a non-asymptotic guarantee for estimation error.
\begin{theorem} \label{theorem 2}
Under Assumptions \ref{assumption 1} and \ref{assumption 2}, take the regularization parameter $\lambda = \sqrt{ c_1 \tau \log p /n}$ with $\tau >2$. If the sample size satisfies $n \geq 8 c_0 c_1 \tau   s \log p$,  then the sequence $\widehat{\bm \Theta}^{(k)}$ generated by the proposed method satisfies
\begin{equation*}
\big \| \widehat{\bm \Theta}^{(k)}  -  \bm \Theta^{\star} \big \|_{\mathrm{F}}  \leq  \underbrace{ 4 c_0 \big\| \big( \bm \Sigma^{\star}  -  \widehat{\bm \Sigma}  \big)_{S^\star \cup I_p} \big\| }_{\mathrm{Statistical \ error}}  +   \underbrace{ \rho^{k} \big\| \widehat{\bm \Theta}^{(0)}   -    \bm \Theta^{\star} \big\|_{\mathrm{F}} }_{\mathrm{Optimization \ error}},
\end{equation*}
with probability at least $1-4/p^{\tau - 2}$, where $\rho = \frac{2+\sqrt{2}}{4}$, $c_0 = 12 \lambda^2_{\max}( \bm \Theta^\star)$, and $c_1 = \big( 80 \max_i \Sigma_{ii}^\star \big)^2$.
\end{theorem}
Theorem \ref{theorem 2} establishes that the estimation error between the estimated precision matrix and the underlying precision matrix can be upper bounded by two terms: optimization error and statistical error. The optimization error decreases at a linear rate, indicating the progressive refinement of the estimate across iterative stages. In contrast, the statistical error remains independent of $k$ and does not decrease during iterations, capturing the intrinsic statistical rate. Consequently, the estimation error is primarily dominated by the statistical error.

We note that the statements in Theorem \ref{theorem 2} are applicable to estimate both \textit{M}-matrices and diagonally dominant \textit{M}-matrices, \textit{i.e.}, $\mathcal{S} = \mathcal{M}^p$ or $\mathcal{M}_D^p$ in Algorithm \ref{algo-1}. The parameter $\tau$ is user-defined; a larger $\tau$ increases the probability of the claims being true but imposes a stricter requirement on the sample size.

\begin{corollary} \label{corollary-result}
Under the same assumptions and conditions as stated in Theorem \ref{theorem 2}, the sequence $\widehat{\bm \Theta}^{(k)}$ generated by the proposed method satisfies
\begin{equation*}
\big \| \widehat{\bm \Theta}^{(k)} - \bm \Theta^{\star} \big \|_{\mathrm{F}}  \leq \underbrace{ c \sqrt{s \log p/n} }_{\mathrm{Statistical \ error}}  +\underbrace{ \rho^{k} \big\| \widehat{\bm \Theta}^{(0)} - \bm \Theta^{\star} \big\|_{\mathrm{F}} }_{\mathrm{Optimization \ error}},
\end{equation*}
with probability at least $1-4/p^{\tau - 2}$, where $c = 2 c_0 \sqrt{2c_1 \tau}$.
\end{corollary}
Corollary~\ref{corollary-result} reveals that the statistical error follows the order of $\sqrt{s\log p /n}$, aligning with the minimax rate achieved in unconstrained graphical models \citep{ravikumar2011high,loh2015regularized}. However, the impact of the \textit{M}-matrix constraint on the minimax rate for estimating sparse precision matrices remains uncertain. We highlight a recent study \citep{soloff2020covariance} that investigated estimating precision matrices under the nonpositivity constraint without sparsity regularization. The minimax optimal rate under symmetrized Stein loss is established to be $\sqrt{\log p / n}$, while removing the nonpositivity constraint results in a reduced rate of $\sqrt{p /n}$. This finding suggests that the nonpositivity constraint affects the minimax rate, particularly in non-sparse situations.

\vspace{-0.2cm}

\section{Experimental Results}\label{sec-experiment} 

We conduct experiments on synthetic and financial time-series data. The code of our method is publicly available at {\url{https://github.com/jxying/ddmtp2}}.

State-of-the-art methods for comparison include \textsf{GLASSO} \citep{friedman2008sparse}, \textsf{CLIME} \citep{cai2011constrained}, \textsf{GSCAD} \citep{loh2015regularized}, \textsf{GGL} \citep{egilmez2017graph}, \textsf{DDGL} \citep{egilmez2017graph}, \textsf {SLTP} \citep{wang2020learning}, and GOLAZO \citep{lauritzen2022locally}. \textsf{GLASSO}, \textsf{CLIME} and \textsf{GSCAD} focus on estimating precision matrices for general graphical models. The first two employ the $\ell_1$-penalty, while the latter utilizes the SCAD penalty. In contrast, \textsf{GGL}, \textsf{DDGL}, and \textsf{SLTP} are specifically designed for MTP\textsubscript{2} graphical models. Both \textsf{GGL} and \textsf{DDGL} use the $\ell_1$-norm regularization, aiming for estimating \textit{M}-matrices and diagonally dominant \textit{M}-matrices, respectively. Meanwhile, \textsf{SLTP} concentrates on learning graph structures rather than estimating precision matrices. Lastly, \textsf{GOLAZO} is tailored for estimating positive graphical models.

\begin{figure}[h]
    \captionsetup[subfigure]{justification=centering}
    \centering
    \subfloat[]{
        \centering
        \includegraphics[scale=.13]{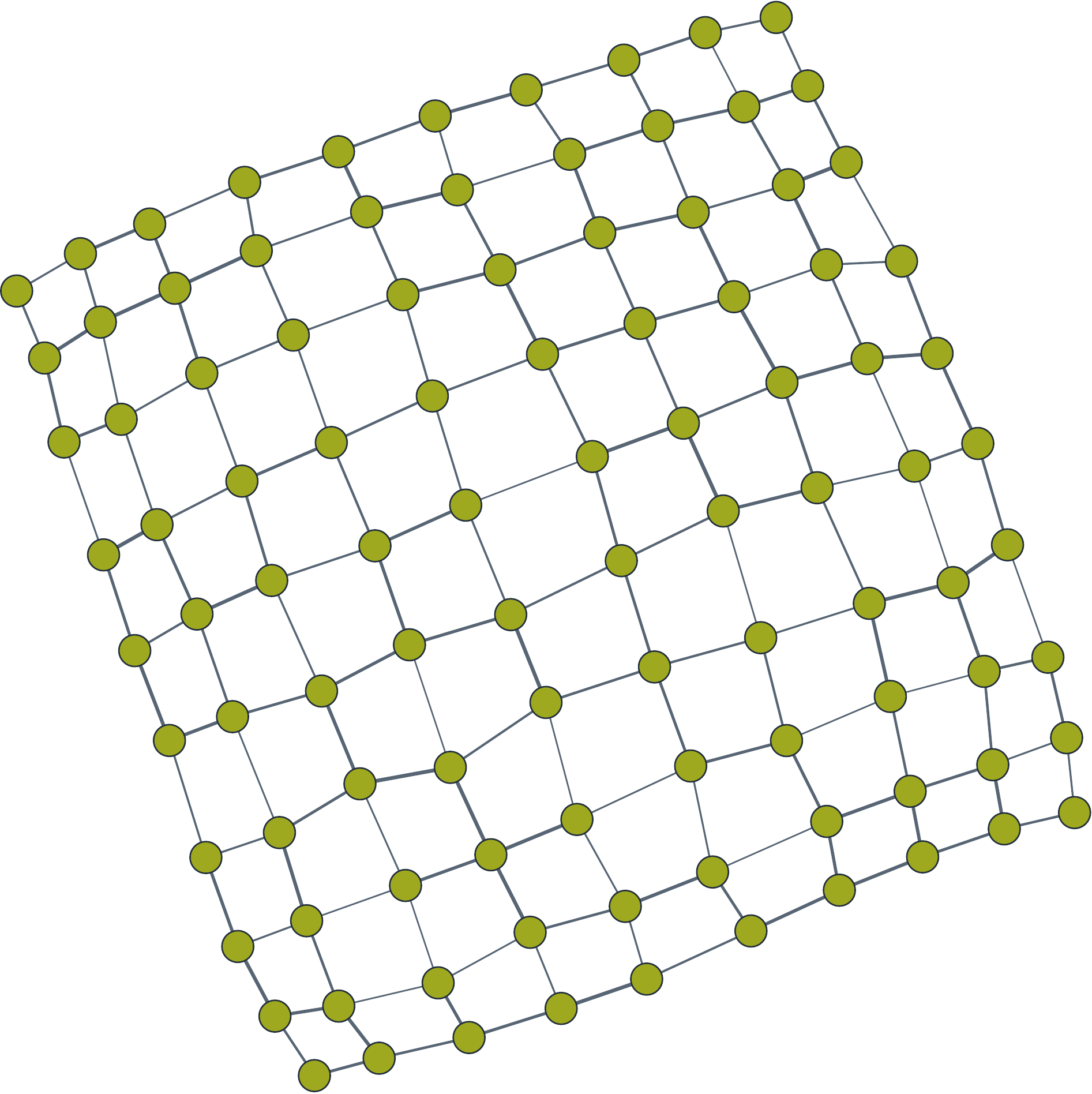}} 
\qquad \quad \
    \subfloat[]{
        \centering
        \includegraphics[scale=.13]{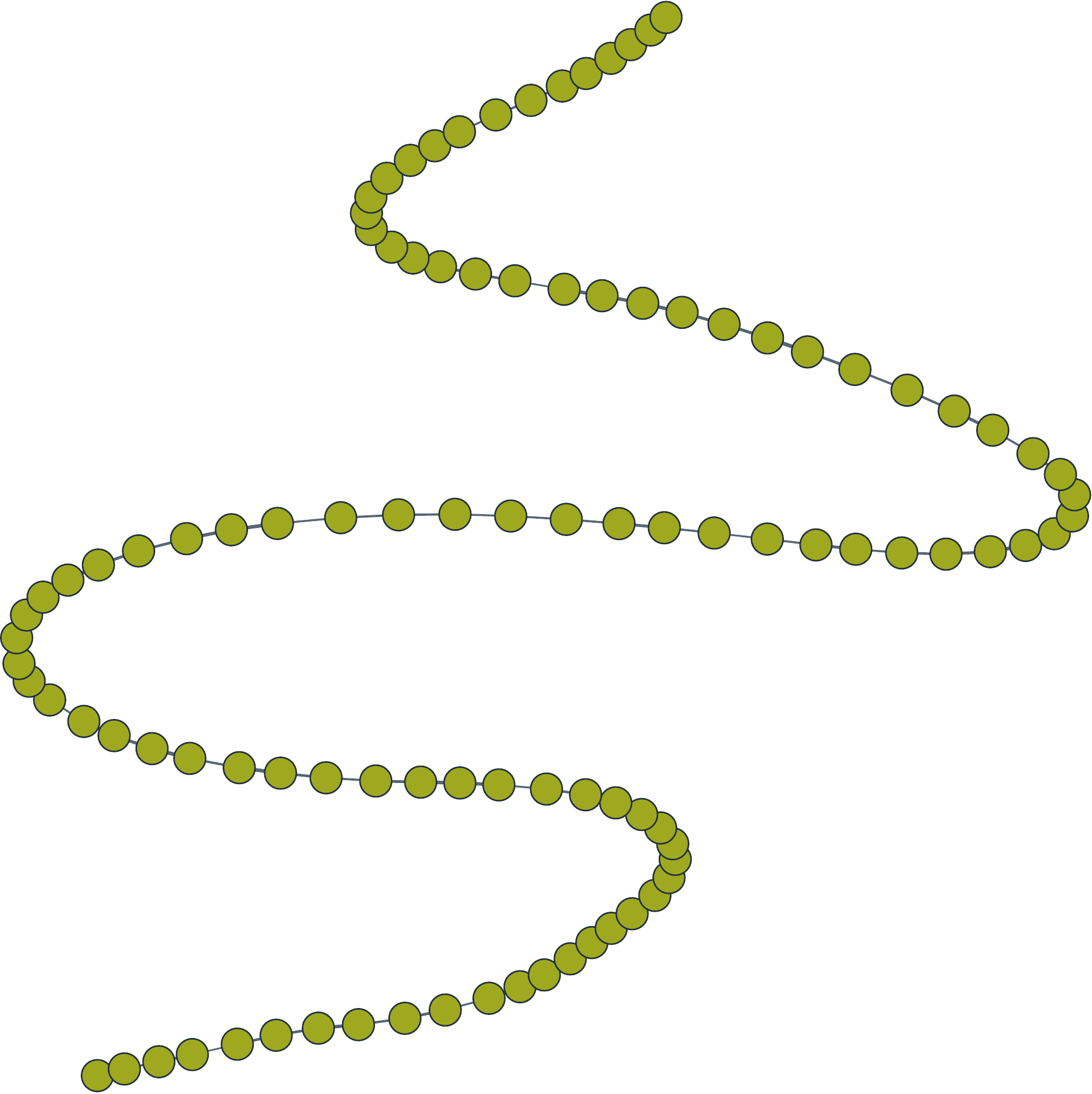}}
\qquad \quad \     
        \subfloat[]{
        \centering
        \includegraphics[scale=.13]{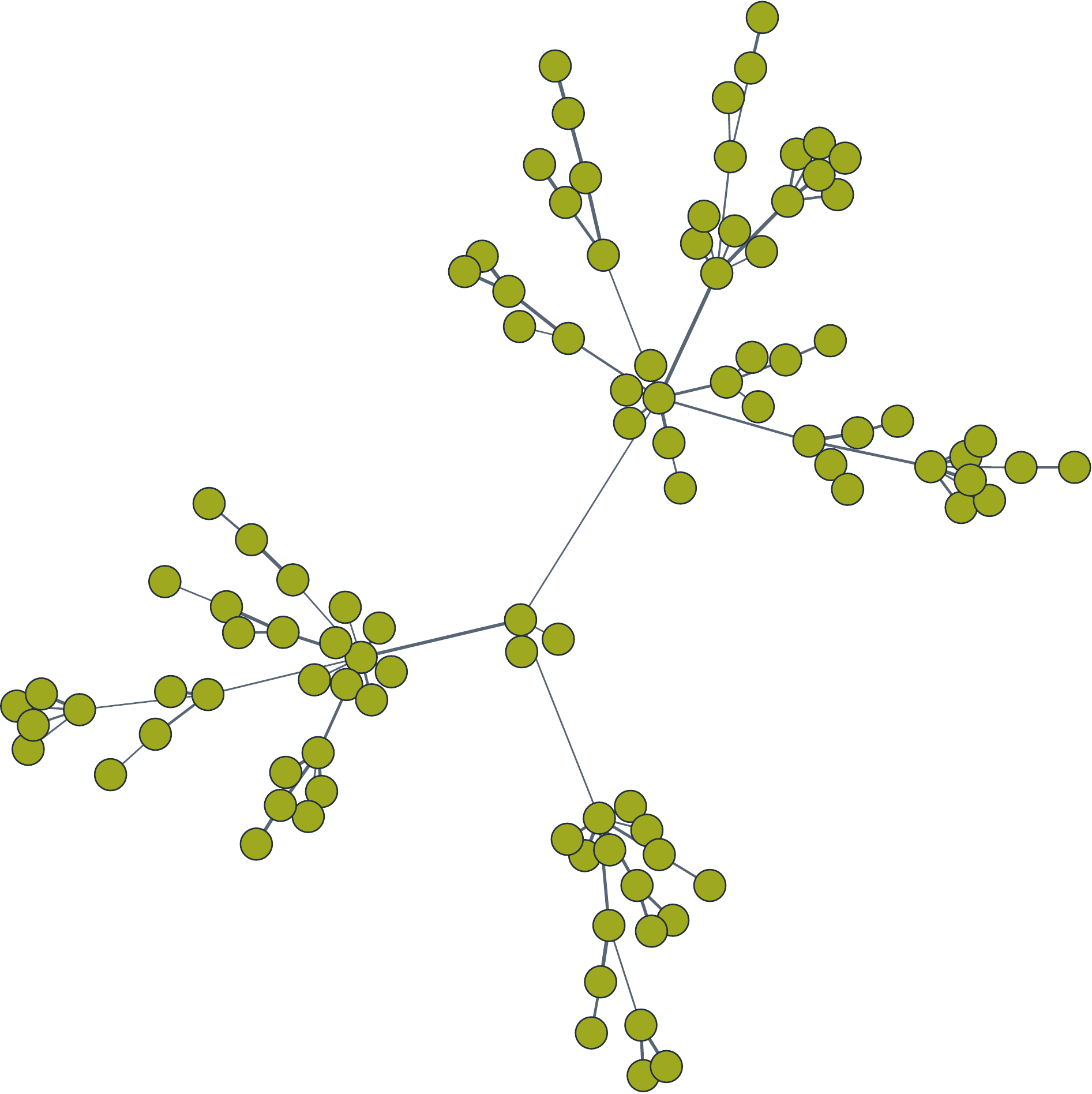}} 
\qquad \quad \
    \subfloat[]{
        \centering
        \includegraphics[scale=.13]{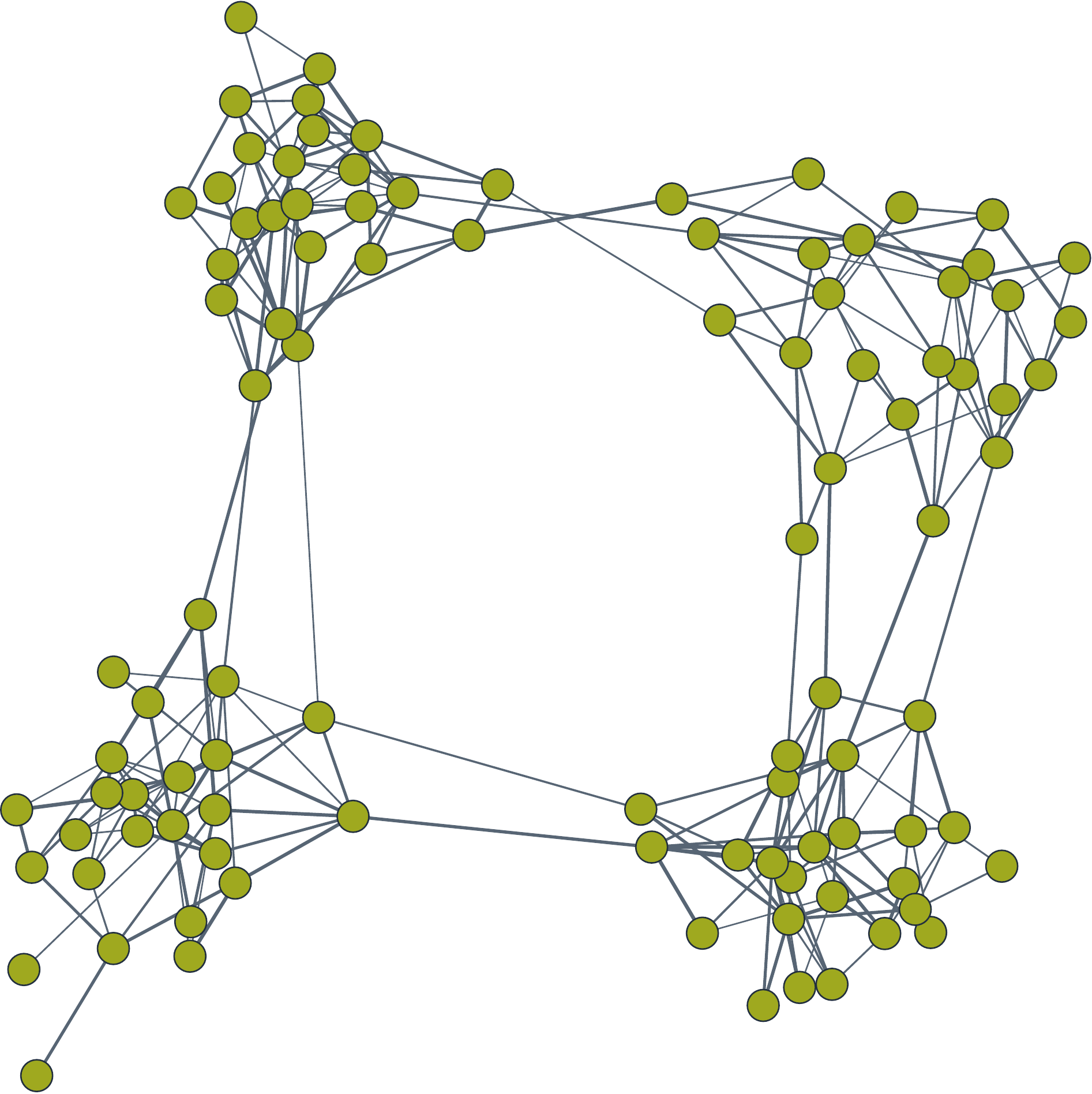}}
    \caption{Illustration of graph structures: (a) grid, (b) line, (c) Barabasi-Albert model, and (d) stochastic block model.}
    \label{fig:graph structures}
    \vspace{-0.2cm}
\end{figure} 

\vspace{-0.2cm}

\subsection{Synthetic Data}\label{sec:synthetic}

We examine four prevalent graph structures in synthetic experiments: grid, line, Barabasi-Albert model \citep{barabasi1999emergence}, and stochastic block model \citep{holland1983stochastic}, as depicted in Figure \ref{fig:graph structures}. Estimating these graph structures is essential in various applications, including social network analysis, image segmentation, and community detection. The graph structures are generated randomly, with each edge assigned a random weight. To assess the performance of precision matrix estimation, we employ estimation error as a metric. For evaluating graph edge selection accuracy, we utilize true/false positive rate and $\mathrm{F}$-$\mathrm{score}$. Further details about experimental settings, including performance measure definitions and synthetic data generation, can be found in Appendix~\ref{appendix-exp}.

 \begin{figure}[t]
    \centering
    \subfloat{
        \centering
        \includegraphics[scale=.245]{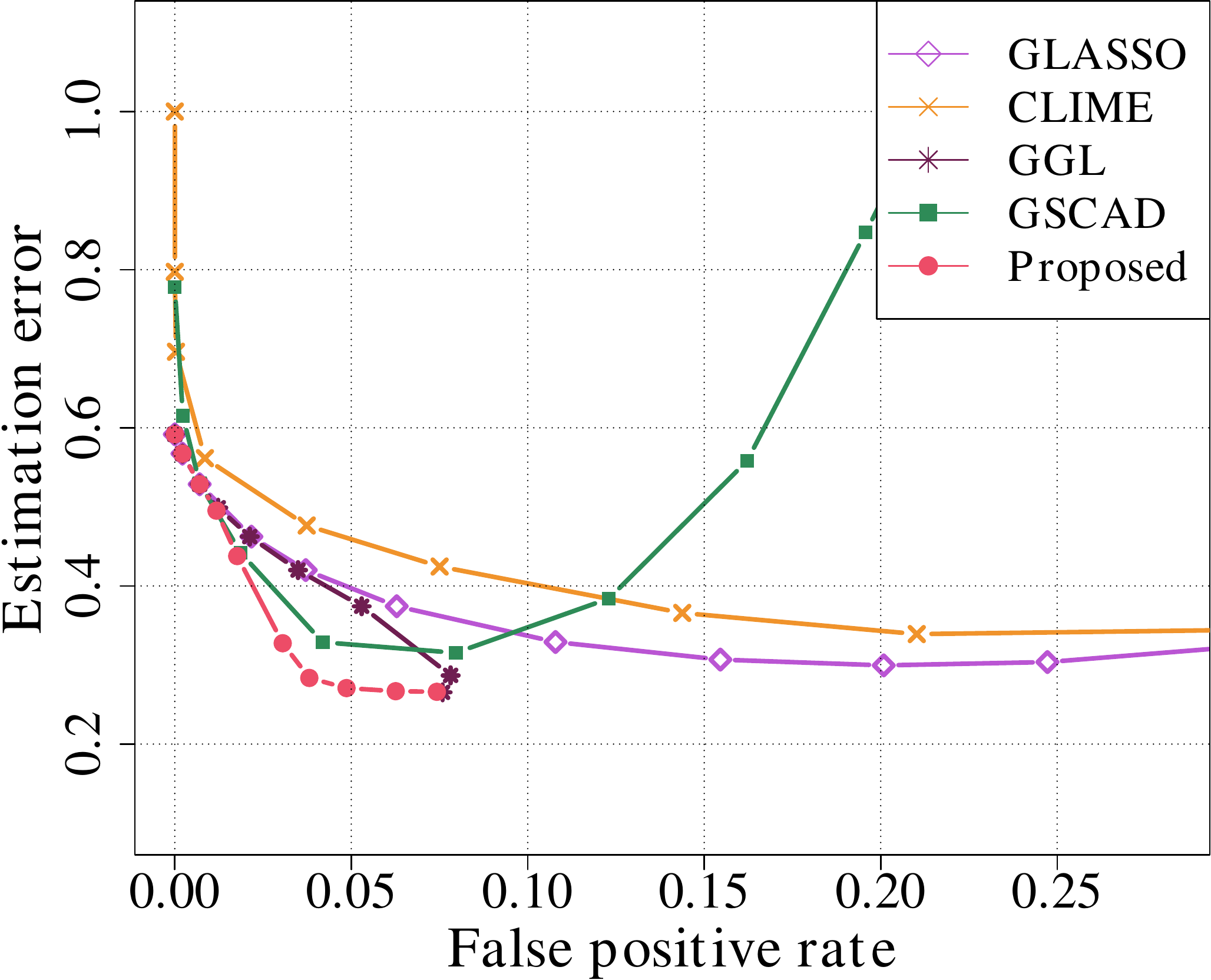}}
\qquad \qquad   
    \subfloat{
        \centering
        \includegraphics[scale=.245]{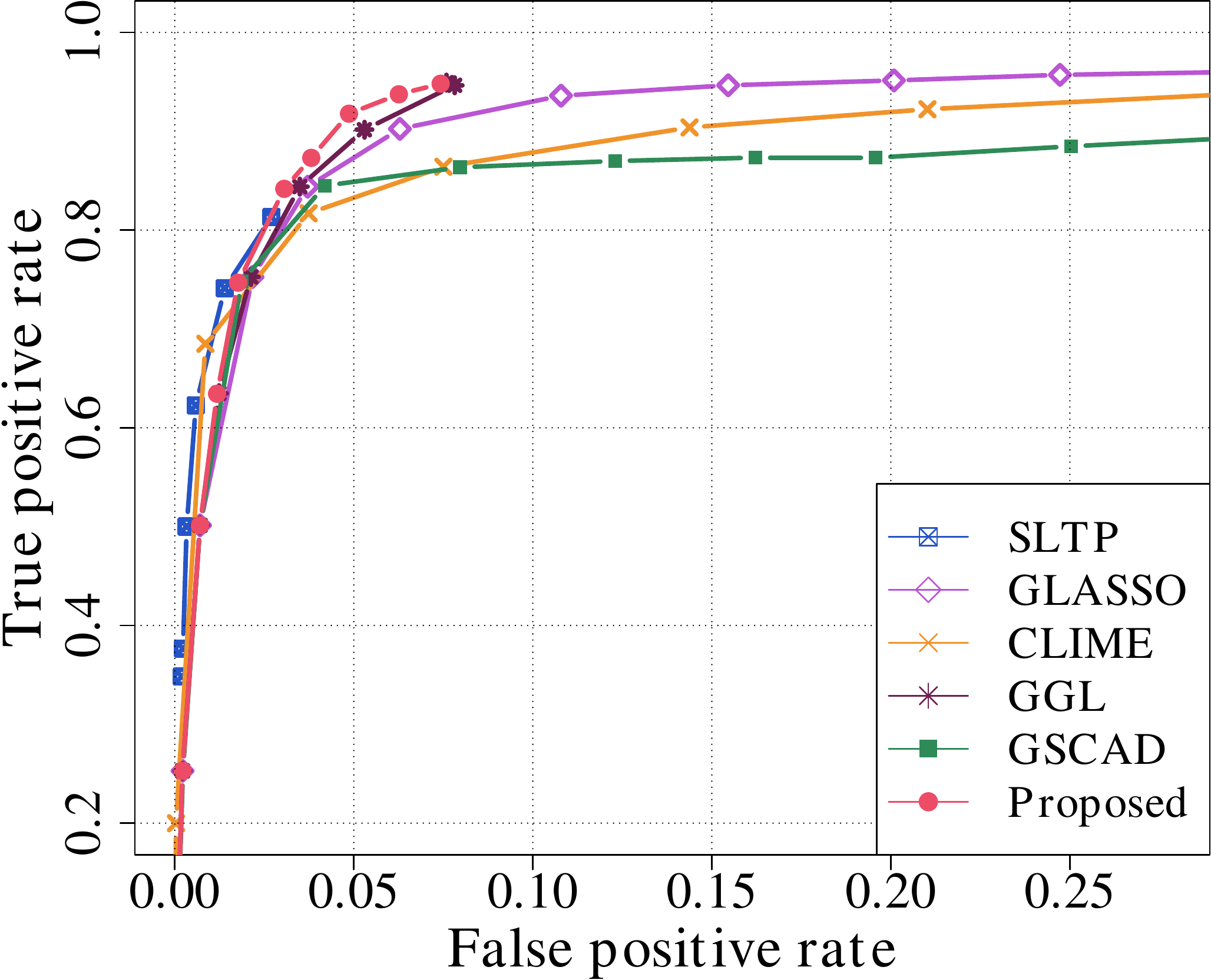}}%
    \caption{Estimation errors and true positive rates versus false positive rates on the grid graph with $\bm \Theta^\star \in \mathcal{M}^p$. The number of nodes $p=100$, and the sample size $n=100$. }
    \label{fig:er-roc-grid}
\end{figure}

 \begin{figure}[t]
    \centering
    \subfloat{
        \centering
        \includegraphics[scale=.245]{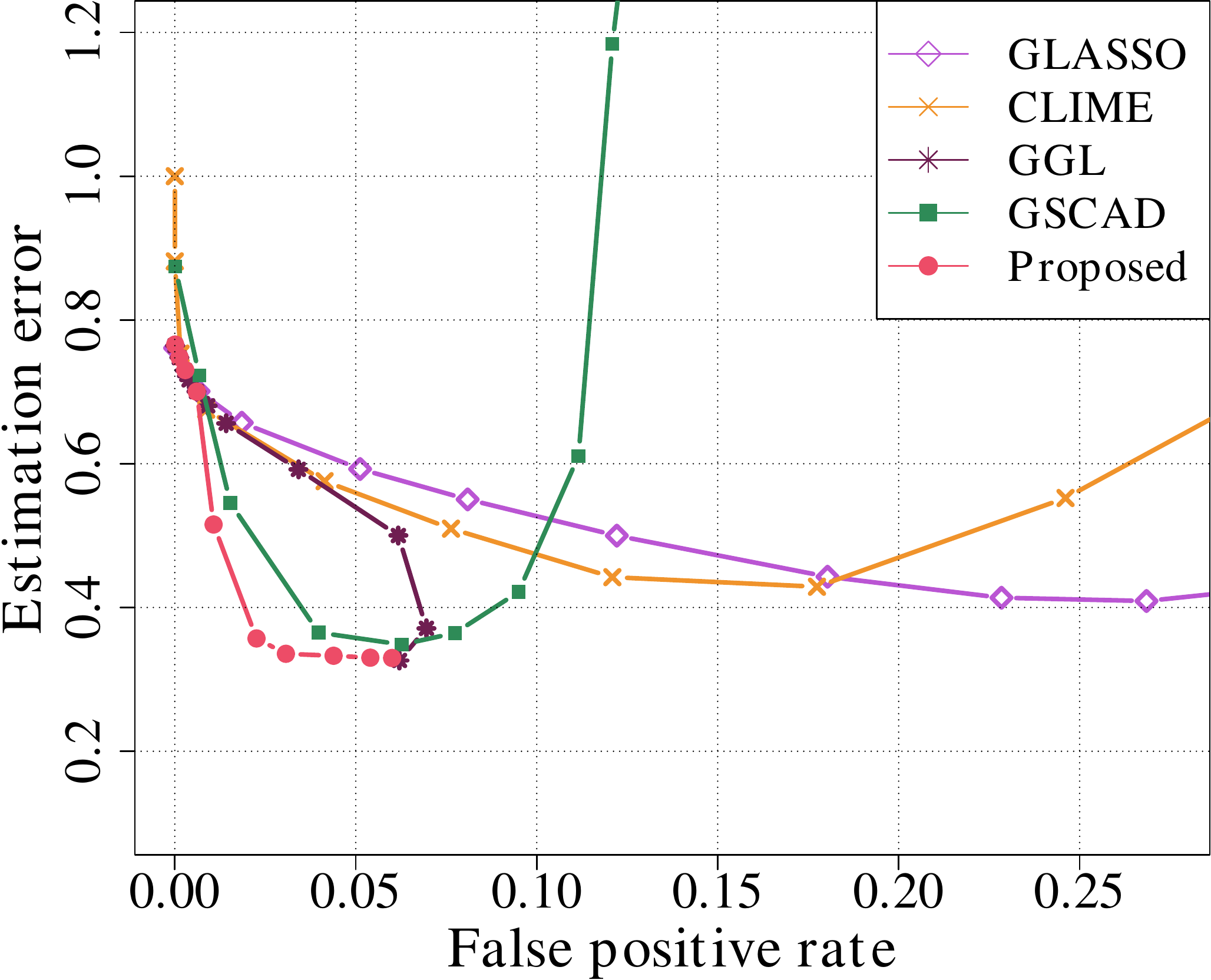}}
 \qquad \qquad 
    \subfloat{
        \centering
        \includegraphics[scale=.245]{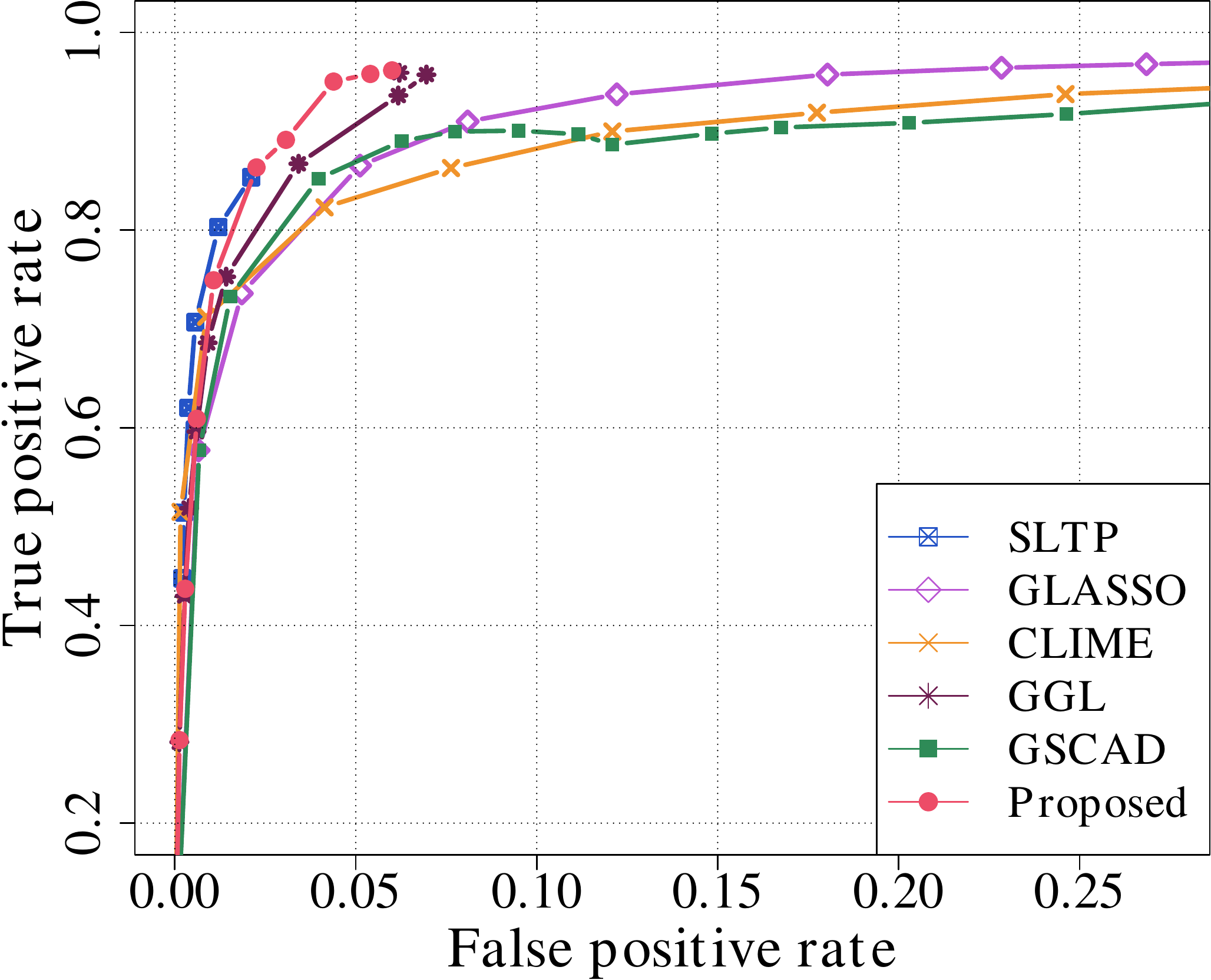}}%
    \caption{Estimation errors and true positive rates versus false positive rates on the line graph with $\bm \Theta^\star \in \mathcal{M}^p$. The number of nodes $p=100$, and the sample size $n=40$.}
    \label{fig:er-roc-line}
\end{figure}

Figure \ref{fig:er-roc-grid} presents estimation errors and true positive rates of various methods as a function of false positive rates on a grid graph, given the number of nodes ($p=100$), sample size ($n=100$), and underlying precision matrix $\bm \Theta^\star \in \mathcal{M}^p$. The curve depicting true positive rates against false positive rates in Figure \ref{fig:er-roc-grid} is known as Receiver Operating Characteristic (ROC) curve, which is generated by setting a range of regularization parameter values from $10^{-3}$ to $1$ for each method, except for \textsf{SLTP}. Instead, \textsf{SLTP} uses a parameter $\gamma$ to control the sparsity of the learned graph, with values ranging from 0.75 to 0.98. The curves in both Figures~\ref{fig:er-roc-grid} and \ref{fig:er-roc-line} are averaged over $50$ realizations.

Figure \ref{fig:er-roc-grid} shows that the proposed method attains a low estimation error and a high true positive rate simultaneously while maintaining a small false positive rate. In comparison, \textsf{GLASSO} and \textsf{CLIME} exhibit the best performance in estimation error and true positive rate within a region of high false positive rates. A higher true positive rate coupled with a lower false positive rate signifies better performance in identifying underlying graph edges. Our method surpasses \textsf{GSCAD} in achieving a smaller estimation error and a higher true positive rate, as we incorporate more prior knowledge of the constraint set $\mathcal{M}^p$. Although \textsf{GGL} also employs $\mathcal{M}^p$, our method achieves a better performance because it refines the \textsf{GGL} estimate in subsequent stages. The comparison with \textsf{SLTP} regarding estimation errors is excluded. It is noteworthy that \textsf{GGL}, \textsf{SLTP}, and our method all possess an upper bound on false positive rates.

 \begin{figure}[t]
    \centering
    \subfloat{
        \centering
        \includegraphics[scale=.24]{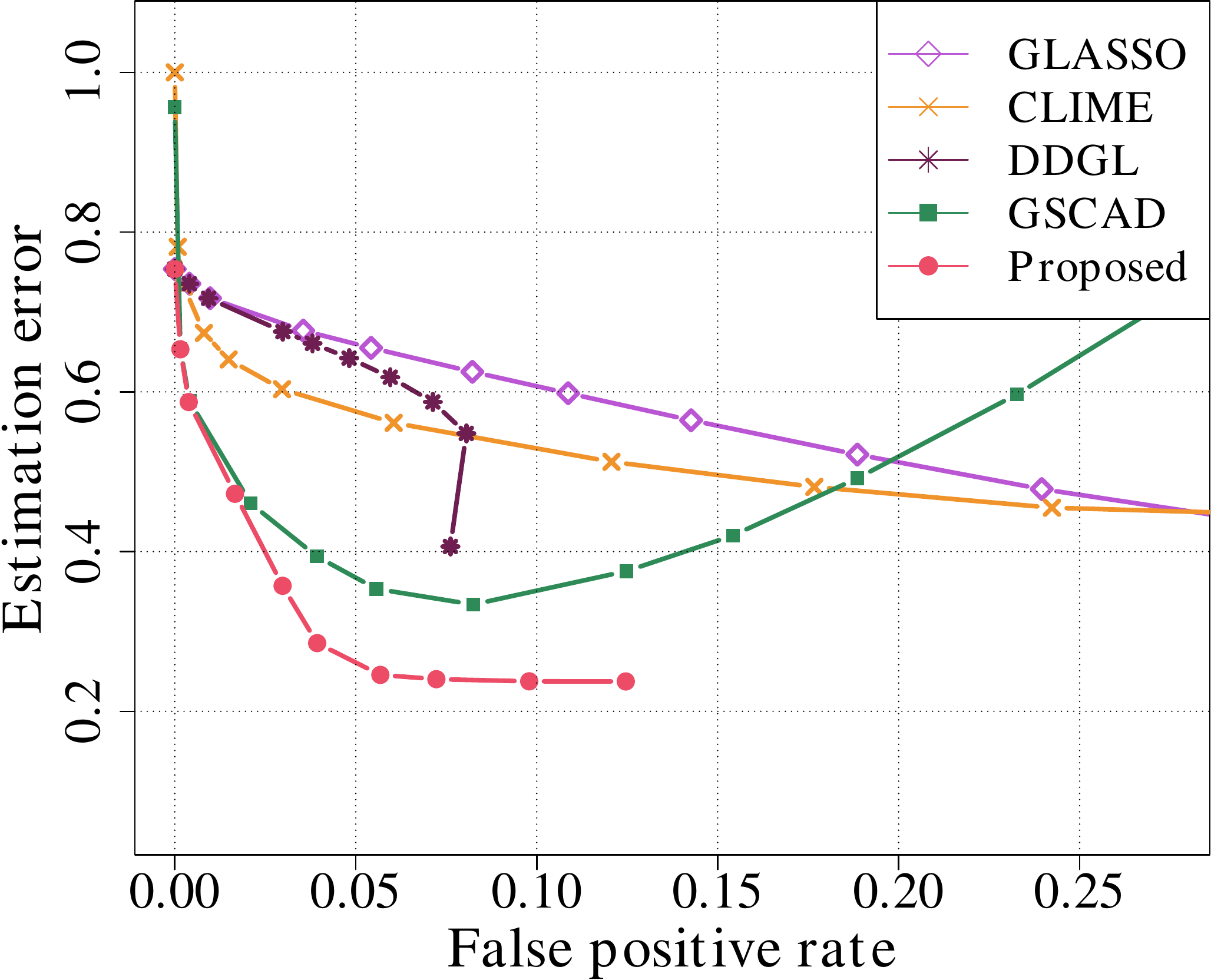}}
   \qquad \qquad 
     \subfloat{
        \centering
        \includegraphics[scale=.24]{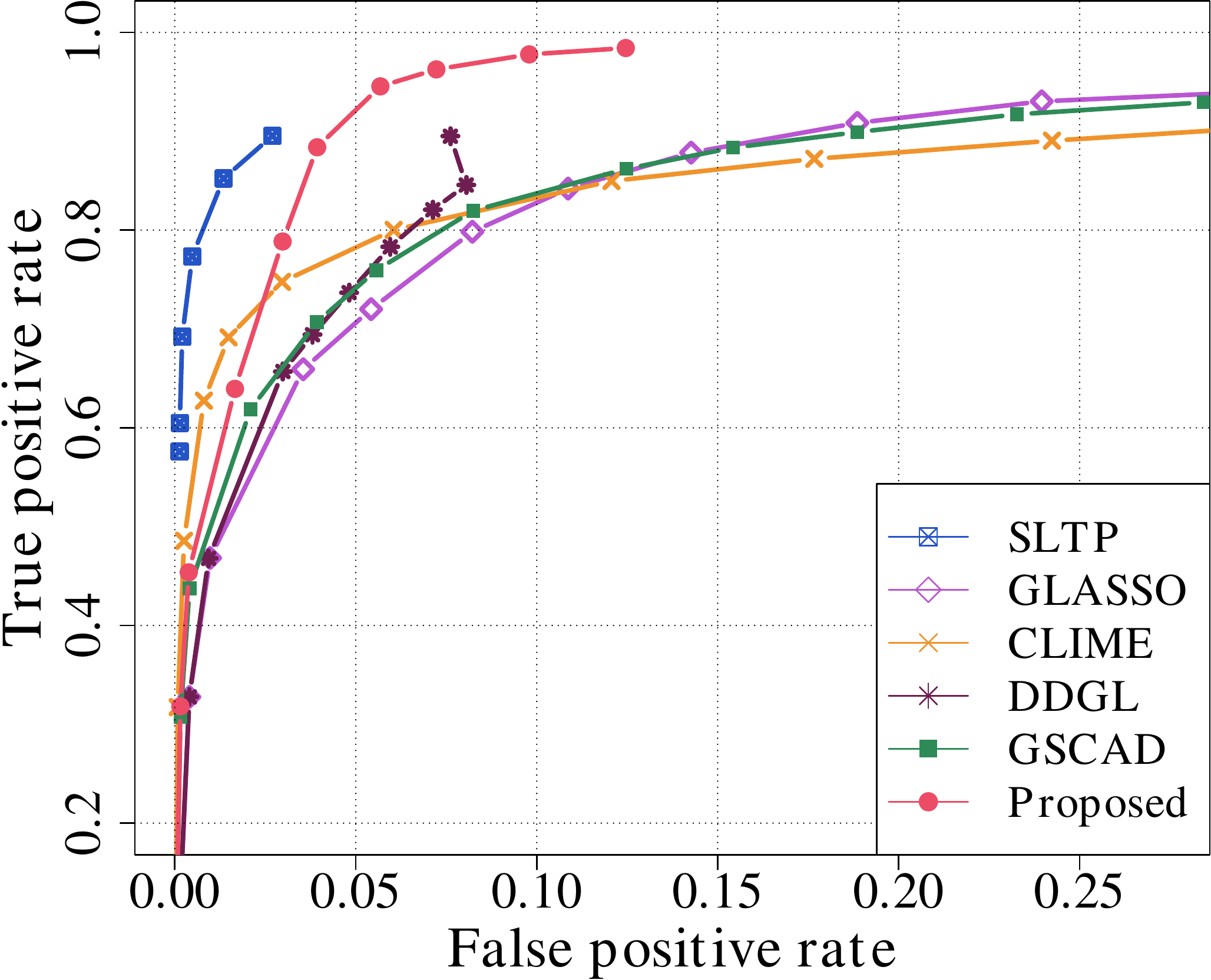}}%
    \caption{Estimation errors and true positive rates versus false positive rates for various methods on Barabasi-Albert model with $\bm \Theta^\star \in \mathcal{M}_D^p$, where $p=100$ and $n=100$.}
    \label{fig:er-roc-tree}
     \vspace{-0.1cm}
\end{figure}

\begin{figure}[t]
    \centering
    \subfloat{
        \centering
        \includegraphics[scale=.24]{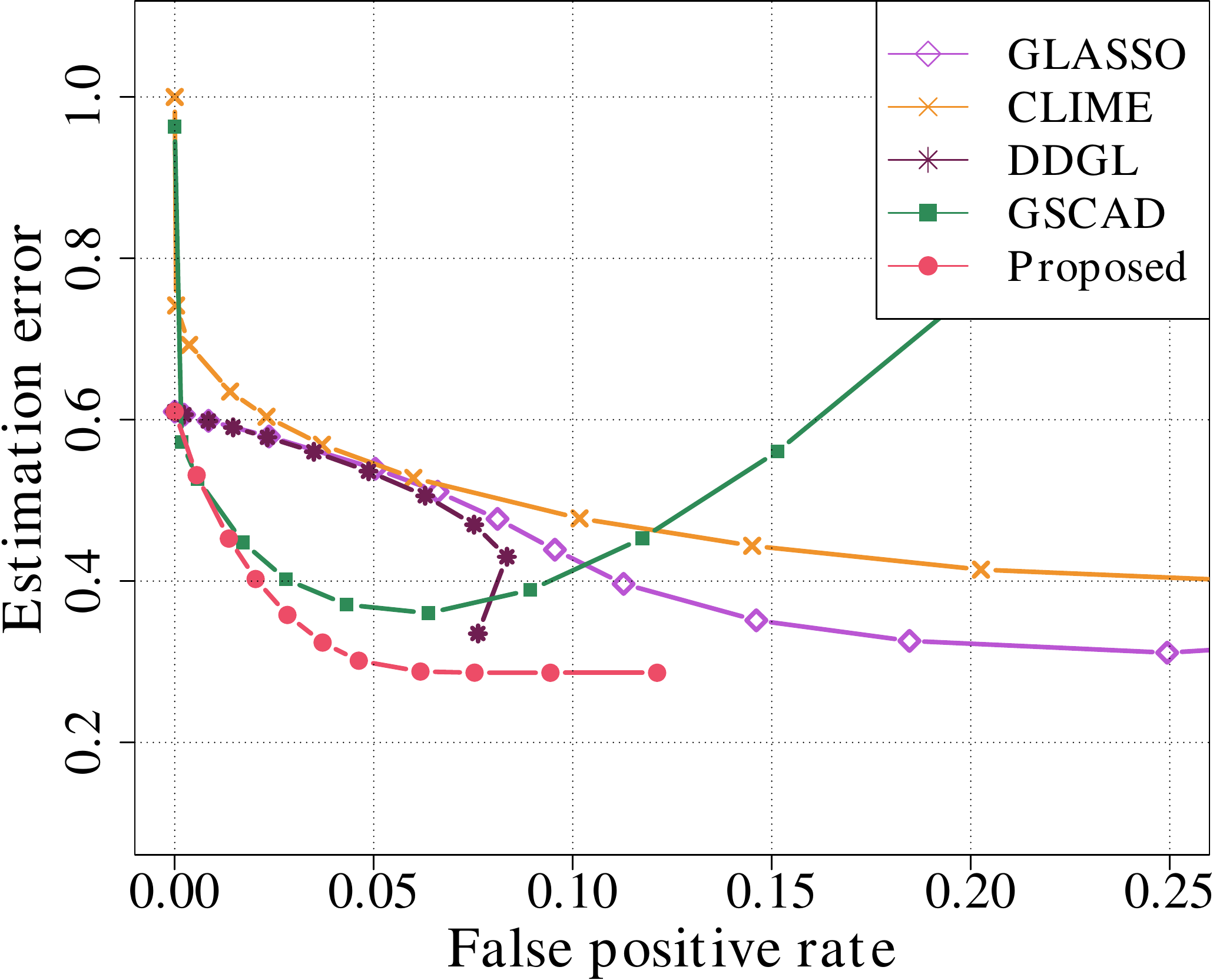}}
     \qquad \qquad 
     \subfloat{
        \centering
        \includegraphics[scale=.24]{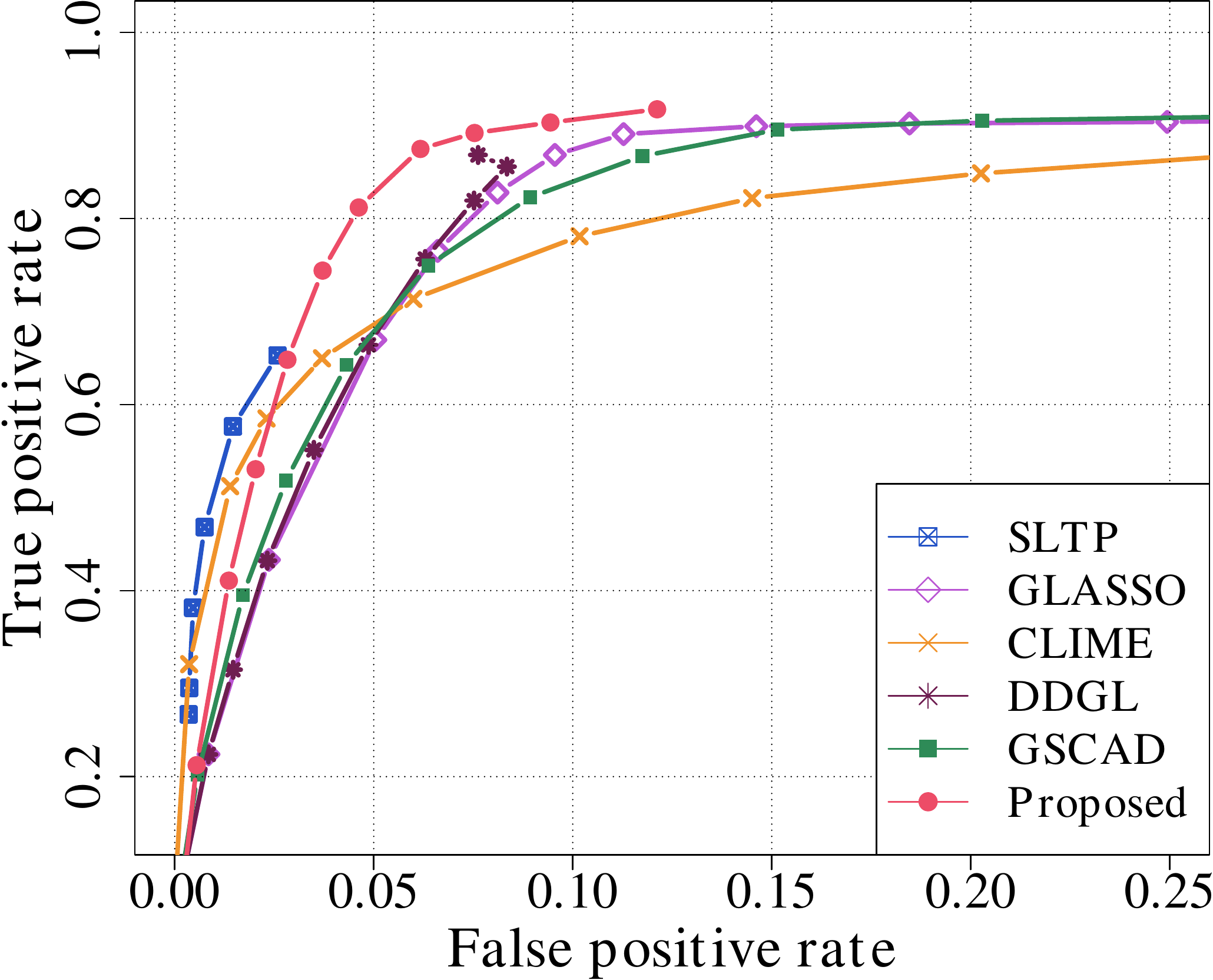}}%
    \caption{Estimation errors and true positive rates versus false positive rates for different methods on the stochastic block model with $\bm \Theta^\star \in \mathcal{M}_D^p$, where $p=100$ and $n=100$.}
    \label{fig:er-roc-modular}
     \vspace{-0.1cm}
\end{figure}

 \begin{figure}[!h]
    \captionsetup[subfigure]{justification=centering}
    \centering
    \subfloat{
        \centering
        \includegraphics[scale=.29]{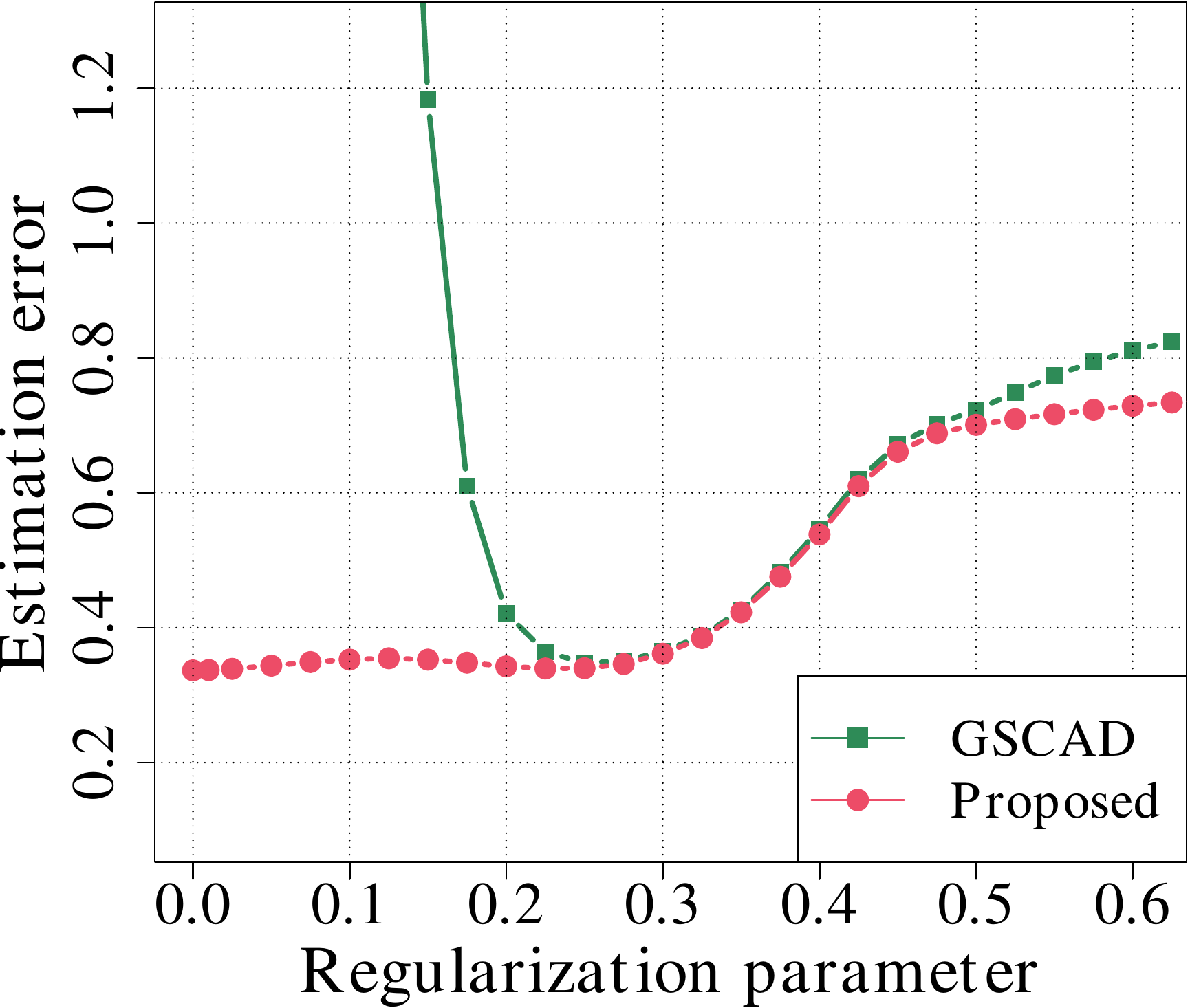}}
\,    
  \qquad \quad \
    \subfloat{
        \centering
        \includegraphics[scale=.29]{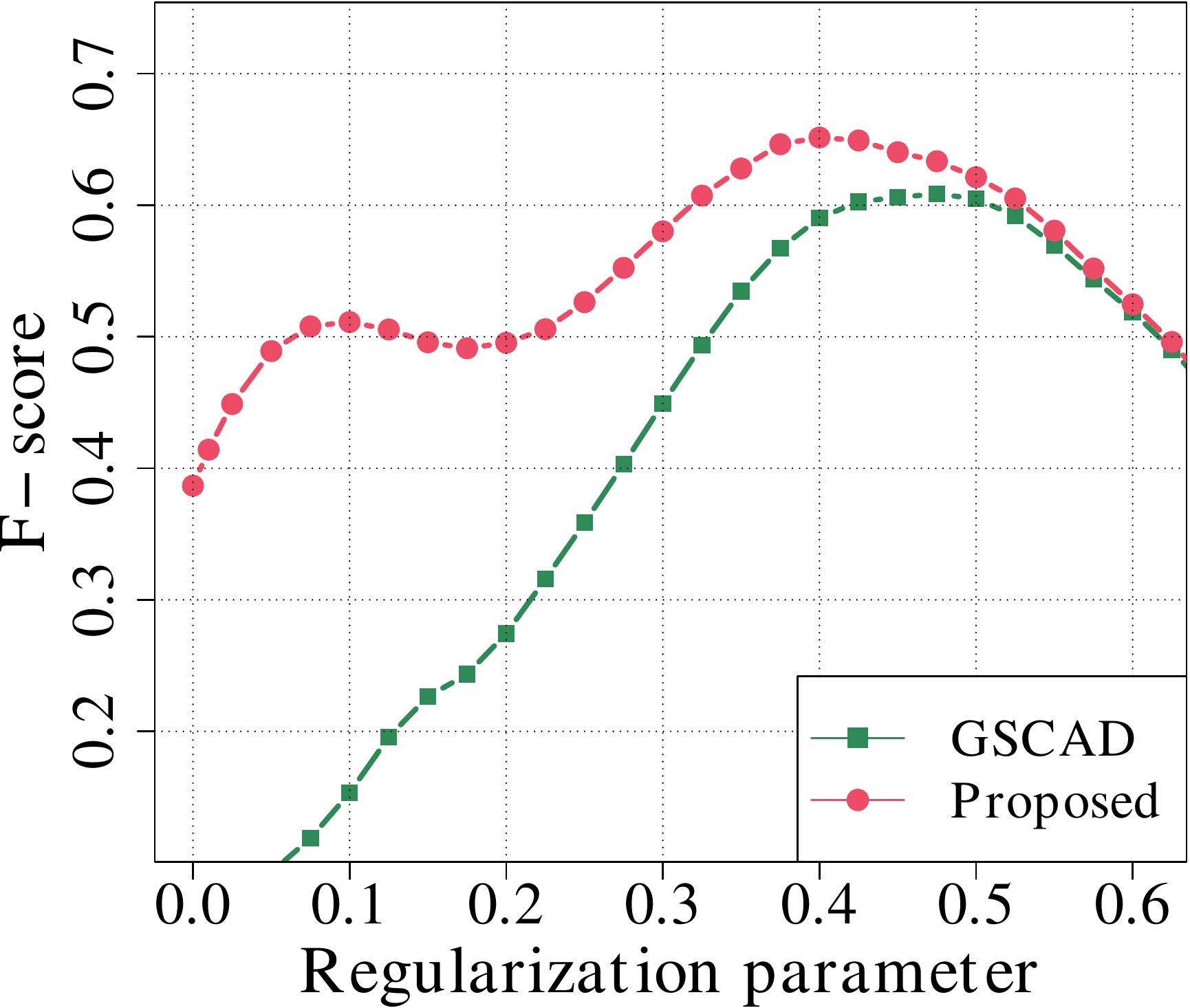}}%
    \caption{Estimation errors and $\mathrm{F}\text{-}\mathrm{scores}$ as a function of regularization parameter values on a line graph.}
    \label{fig:sensitivity-line}
    \vspace{-0.2cm}
\end{figure}

Figure \ref{fig:er-roc-line} displays the estimation errors and true positive rates of different methods as a function of false positive rates on a line graph. Our proposed method consistently achieves smaller estimation errors and higher true positive rates compared to state-of-the-art methods in the region of low false positive rates. This region is of primary interest, as an ideal estimate should exhibit low estimation error while maintaining high true positive rate and low false positive rate. Consequently, our method is capable of simultaneously achieving a reduced precision matrix estimation error and enhanced graph edge selection accuracy.

Figure \ref{fig:er-roc-tree} presents comparisons on the Barabasi-Albert model with the underlying precision matrix $\bm \Theta^\star \in \mathcal{M}_D^p$. Our method consistently achieves lower estimation errors and higher true positive rates compared to state-of-the-art methods, highlighting its superior performance in both precision matrix estimation and graph edge identification. \textsf{SLTP} attains a higher true positive rate than all other methods when the false positive rate is below 0.05; however, its highest true positive rate is still lower than that of our method. Similar observations can be made for the stochastic block model, as displayed in Figure \ref{fig:er-roc-modular}. The curves in Figures~\ref{fig:er-roc-tree} and \ref{fig:er-roc-modular} are averaged over $50$ realizations.

Figure~\ref{fig:sensitivity-line} illustrates the robustness of \textsf{GSCAD} and our method regarding regularization parameter choices in terms of estimation error and $\mathrm{F}\text{-}\mathrm{score}$. Experiments are conducted on a line graph with $p=100$ and $n=40$. Figure~\ref{fig:sensitivity-line} reveals that both estimation errors and $\mathrm{F}\text{-}\mathrm{scores}$ of our method are more stable across varying regularization parameter values than those of \textsf{GSCAD}. Notably, the estimation error of \textsf{GSCAD} grows rapidly as the regularization parameter decreases, while that of our method remains steady. Thus, our method exhibits less sensitivity to the regularization parameter selection than \textsf{GSCAD}.

The key difference between our method and \textsf{GSCAD} is the former imposes an additional \textit{M}-matrix constraint, while the latter does not. Thus, our method searches for a solution within a smaller space than \textsf{GSCAD}, excluding some stationary points that yield large estimation errors. This is why our method is more robust to regularization parameter selection. In fact, an extremely small regularization parameter can render the optimization problem for \textsf{GSCAD} ill-defined, meaning the minimum is not finite.

\begin{table}[h!]
\centering
\setlength{\tabcolsep}{13pt}
\begin{tabular}{@{}lccccc@{}}
\toprule
$n/p$ & \textsf{GLASSO} & \textsf{CLIME} & \textsf{GSCAD} & \textsf{GGL}   & \textsf{Proposed} \\ \midrule
1                 & 0.421  & 0.472 & 0.328 & 0.403 & \textbf{0.267}    \\
2                 & 0.366  & 0.423 & 0.212 & 0.344 & \textbf{0.171}    \\
5                 & 0.358  & 0.242 & 0.110 & 0.297 & \textbf{0.100}    \\ \bottomrule
\end{tabular}
\caption{Estimation errors for various methods across varying sample size ratios, with a consistently low false positive rate (FPR < 0.05).}
\label{tab:sample_size_performance}
\vspace{-0.1cm}
\end{table}

Table~\ref{tab:sample_size_performance} displays the estimation errors for various methods under varying sample size ratios $n/p$, maintaining a low false positive rate (FPR < 0.05). Our method consistently achieves lower estimation errors compared to state-of-the-art methods across a range of sample size ratios. Notably, the superior performance of our method becomes more evident in cases with smaller sample size ratios, indicative of high-dimensional scenarios.

\subsection{Financial Time-series Data}
In this subsection, we present comparisons of various methods on the financial time-series data, where the MTP\textsubscript{2} assumption has been well-justified due to the market factor causing positive dependence among stocks \citep{Agrawal20,cardoso2020algorithms}. We collect data from 204 stocks composing the S\&P 500 index for the period between January 5th, 2004, and December 30th, 2006, resulting in 503 observations per stock (\textit{i.e.}, $p=204$ and $n=503$). We construct the log-returns data matrix $\bm X \in \mathbb{R}^{n \times p}$ as follows:
\begin{equation}
X_{i,j} = \log P_{i,j} - \log P_{i-1, j},
\end{equation}
where $P_{i,j}$ represents the closing price of the $j$-th stock on the $i$-th day. We then use the sample correlation matrix as input for each method. The 204 stocks are categorized into five sectors according to the Global Industry Classification Standard (GICS) system: Consumer Staples, Consumer Discretionary, Industrials, Energy, and Information Technology. We compare the proposed method with \textsf{GLASSO}, \textsf{GSCAD}, \textsf{GGL}, and \textsf{GOLAZO} on this financial time-series data.

In contrast to synthetic data, we cannot evaluate estimation error for each method, as the underlying precision matrix is unavailable. However, we expect stocks within the same sector to be somewhat interconnected. Consequently, we use the \textit{modularity} \citep{newman2006modularity} to assess the estimation performance. Define a graph $\mathcal{G}=\left( \mathcal{V}, \mathcal{E}, \bm A\right)$, where $\mathcal{V}$ is the vertex set, $\mathcal{E}$ is the edge set, and $\bm A$ is the adjacency matrix. The \textit{modularity} of graph $\mathcal{G}$ is given by:
\begin{equation*}
    Q :=  \frac{1}{2 |\mathcal{E}|} \sum_{i, j \in |\mathcal{V}|}\Big(A_{ij} - \frac{d_id_j}{2|\mathcal{E}|}\Big)\delta(c_i, c_j),
\end{equation*}
where $d_i$ represents the degree of the $i$-th node, $c_i $ denotes the type of the $i$-th node, and $\delta(\cdot, \cdot)$ is the Kronecker delta function with $\delta(a, b) =1$ if $a=b$ and $0$ otherwise. A stock graph with high \textit{modularity} exhibits dense connections between stocks within the same sector and sparse connections between stocks from different sectors. Therefore, a higher \textit{modularity} value indicates a better representation of the actual stock network.

 \begin{figure*}[htb]
    \captionsetup[subfigure]{justification=centering}
    \centering
    \begin{subfigure}[t]{0.2\textwidth}
        \centering
        \includegraphics[scale=.305]{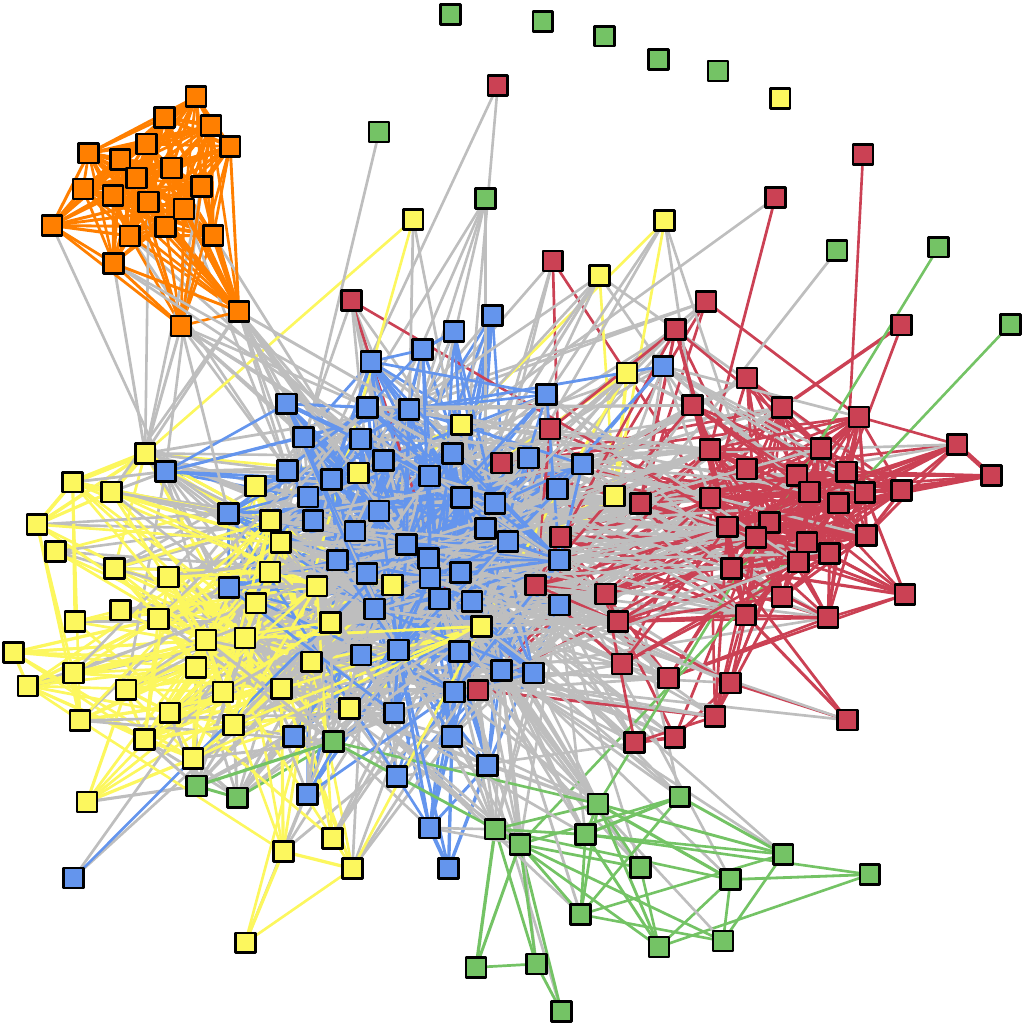}
        \caption{\textsf{GLASSO}}
    \end{subfigure}%
        \begin{subfigure}[t]{0.2\textwidth}
        \centering
        \includegraphics[scale=.305]{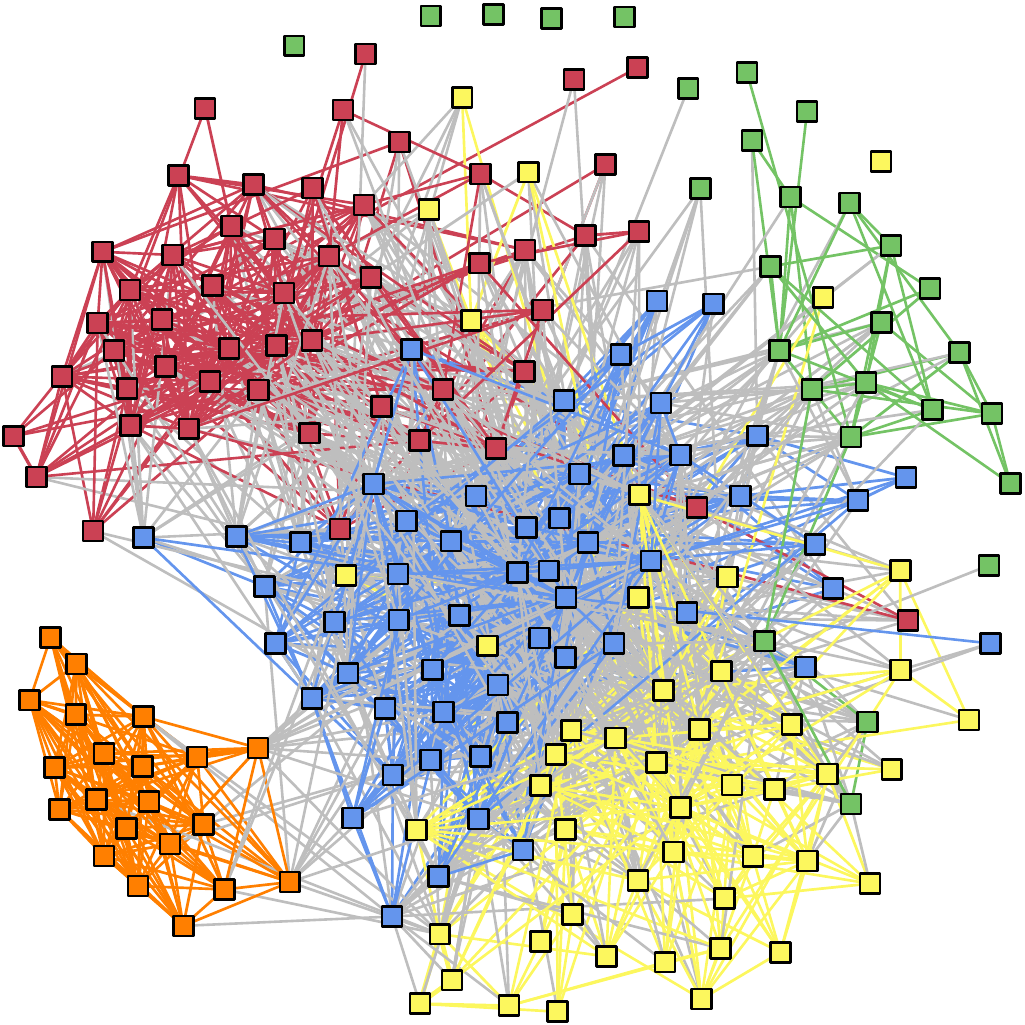}
        \caption{\textsf{GSCAD}}
    \end{subfigure}%
    \begin{subfigure}[t]{0.2\textwidth}
        \centering
        \includegraphics[scale=.305]{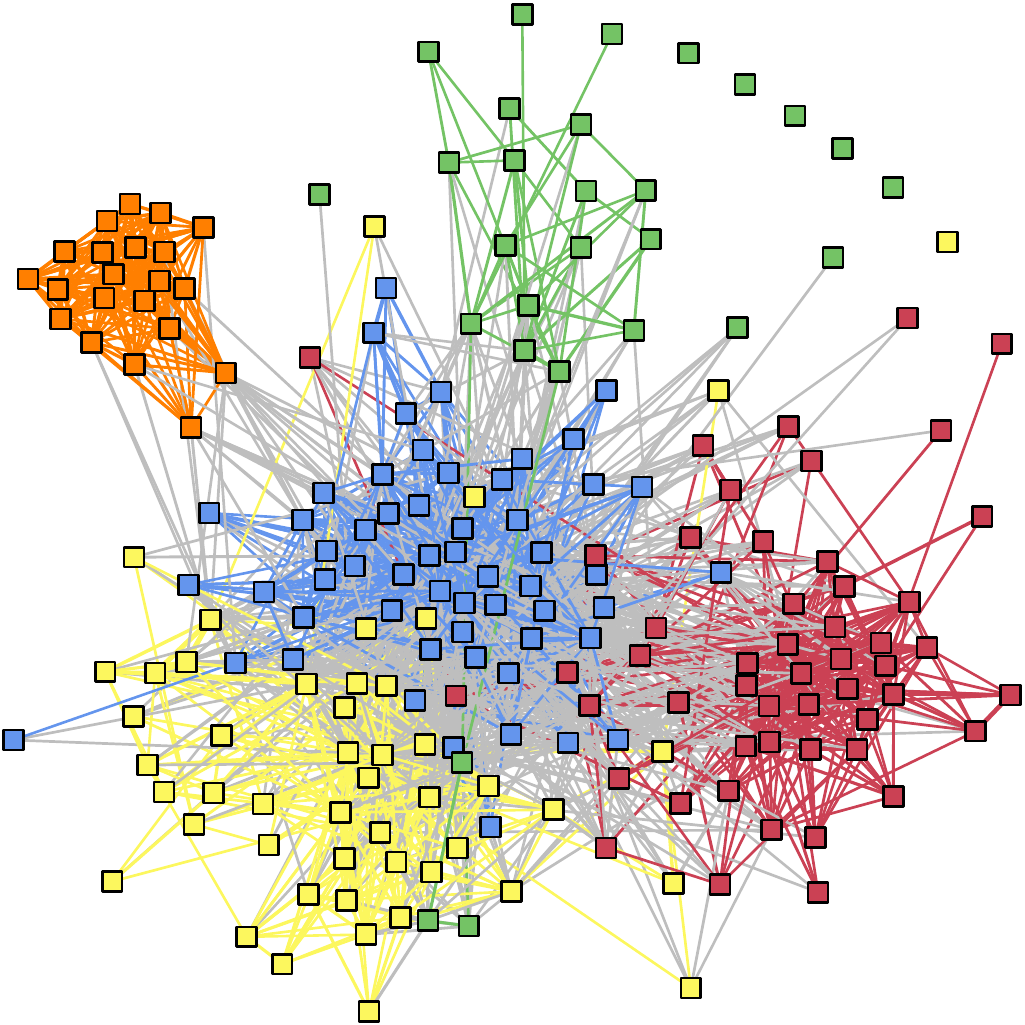}
        \caption{\textsf{GGL}}
    \end{subfigure}%
        \begin{subfigure}[t]{0.2\textwidth}
        \centering
        \includegraphics[scale=.305]{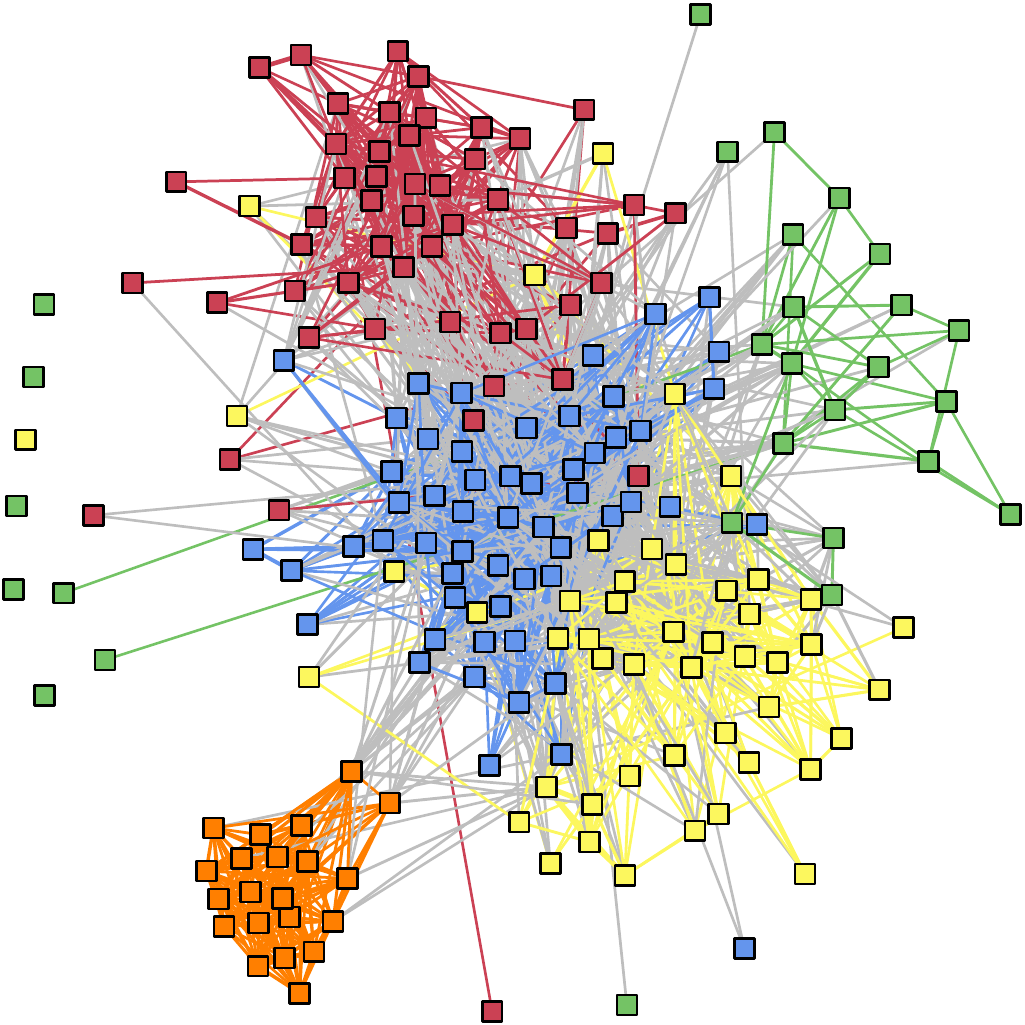}
        \caption{\textsf{GOLAZO}}
    \end{subfigure}%
    \begin{subfigure}[t]{0.2\textwidth}
        \centering
        \includegraphics[scale=.305]{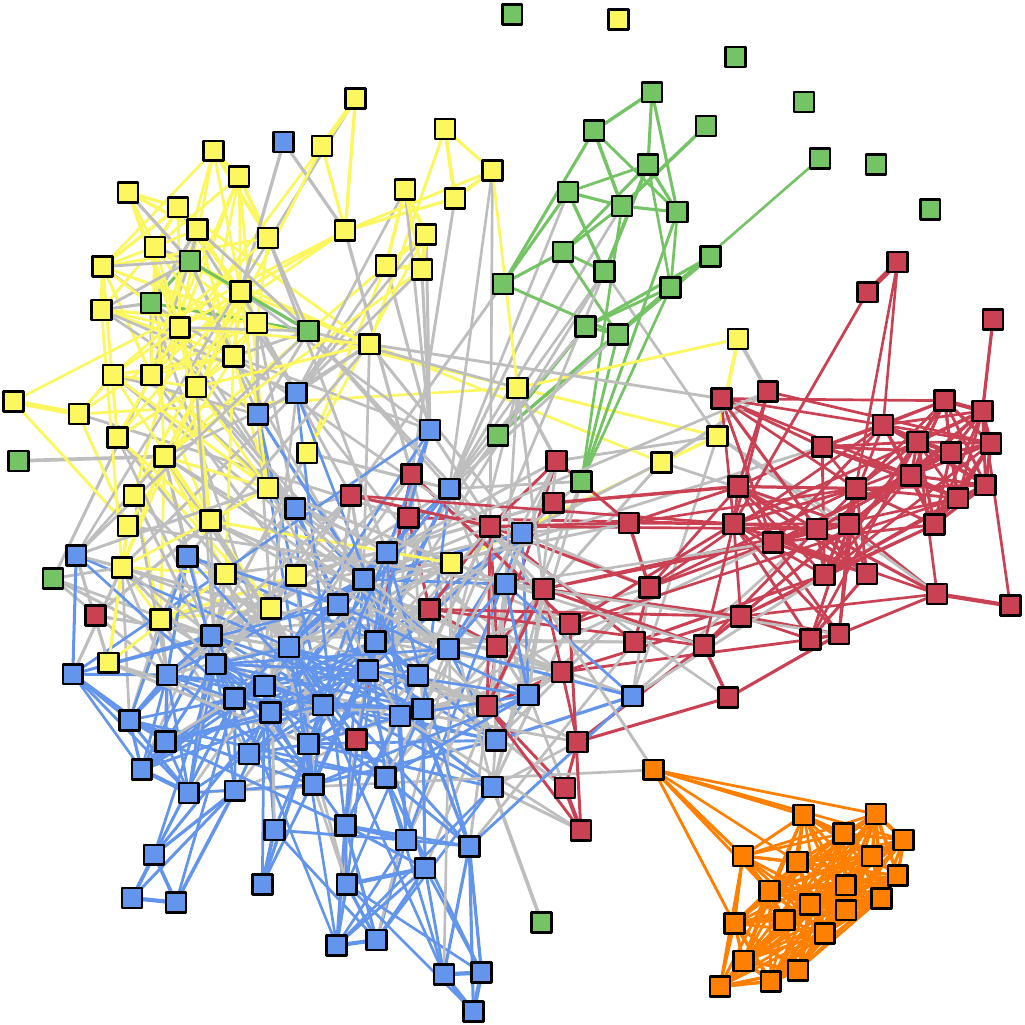}
        \caption{\textsf{Proposed}}
    \end{subfigure}%
    \caption{Stock graphs characterized by the precision matrices estimated via (a) \textsf{GLASSO}, (b) \textsf{GSCAD}, (c) \textsf{GGL}, (d) \textsf{GOLAZO}, and (e) the proposed method. The \textit{modularity} values for \textsf{GLASSO}, \textsf{GSCAD}, \textsf{GGL}, \textsf{GOLAZO}, and our method are 0.3534, 0.3433, 0.3530, 0.3578, and 0.4978, respectively. We fine-tune the regularization parameter for each method based on the \textit{modularity} value while allowing only a few isolated nodes.}
    \label{fig:stock}
\end{figure*}

\begin{table*}[ht]
\centering
\caption{\textit{Modularity} values of S\&P 500 stock graphs learned via various methods using a rolling-window approach, covering the period from January 3rd, 2014, to October 31st, 2017.}
\vspace{-0.1cm}
\setlength{\tabcolsep}{5.05pt}
\begin{tabular}{@{}lccccccccc@{}}
\toprule
Methods & Jan. 2016 & Apr. 2016 & Jul. 2016 & Oct. 2016 & Jan. 2017 & Apr. 2017 & Jul. 2017 & Oct. 2017 \\ \midrule
\textsf{GLASSO}     & 0.4038    & 0.4072    & 0.4115    & 0.3956    & 0.3982    & 0.3910    & 0.3911    & 0.4184    \\
\textsf{GSCAD}      & 0.3880    & 0.3850    & 0.3935    & 0.3803    & 0.3848    & 0.3771    & 0.3822    & 0.3860    \\
\textsf{GGL}        & 0.4088    & 0.4118    & 0.4193    & 0.3991    & 0.3927    & 0.3896    & 0.3845    & 0.4156    \\
\textsf{GOLAZO}     & 0.4068    & 0.4098    & 0.4173    & 0.4007    & 0.4029    & 0.3945    & 0.3953    & 0.4199    \\
\textsf{Proposed}   & \textbf{0.5441}    & \textbf{0.5289}    & \textbf{0.5298}    & \textbf{0.5139}    & \textbf{0.5138}    & \textbf{0.5118}    & \textbf{0.5219}    & \textbf{0.5019}    \\ \bottomrule
\end{tabular}
\label{tab:algo_performance}
\end{table*}

Figure \ref{fig:stock} shows stock graphs based on precision matrices estimated using various methods. It is observed that the performance of the proposed method stands out, as most connections are between nodes within the same sector, and only a few connections (gray-colored edges) are between nodes from distinct sectors, which are often spurious from a practical perspective. We further compare the \textit{modularity} value for each method. The \textit{modularity} values for \textsf{GLASSO}, \textsf{GSCAD}, \textsf{GGL}, \textsf{GOLAZO}, and our method are 0.3534, 0.3433, 0.3530, 0.3578, and 0.4978, respectively, indicating that our method outperforms state-of-the-art methods in representing the stock network.

In the subsequent experiment, we collect data from 82 stocks in the S\&P 500 index, representing three sectors: Communication Services, Utilities, and Real Estate. The data spans a timeframe from January 3rd, 2014, to October 31st, 2017, resulting in 945 observations.

Table~\ref{tab:algo_performance} presents \textit{modularity} values of graphs learned using a rolling-window approach. We select a window length of 504 days (roughly two years of stock market days) and shift it by 63 days (approximately three months of stock market days). Higher \textit{modularity} signifies a more accurate representation of the actual stock network. Table~\ref{tab:algo_performance} demonstrates that the proposed method consistently attains the highest \textit{modularity} values throughout the entire period.

\section{Conclusions}\label{conclusions}

In this paper, we have proposed an adaptive multiple-stage method to estimate both \textit{M}-matrices and diagonally dominant \textit{M}-matrices as precision matrices in MTP\textsubscript{2} Gaussian graphical models. The proposed method starts with an initial estimate obtained via an $\ell_1$-regularized maximum likelihood estimation and refines it in subsequent stages by solving a sequence of adaptive $\ell_1$-regularized problems. We have provided theoretical analysis for estimation error. Experiments on both synthetic and real-world data have shown that our method outperforms state-of-the-art methods in estimating precision matrices and identifying graph edges. We have also demonstrated that incorporating the \textit{M}-matrix constraint improves robustness regarding regularization parameter selection.

\section*{Acknowledgements}

This work was supported by the Hong Kong Research Grants Council GRF 16207820, 16310620, and 16306821, the Hong Kong Innovation and Technology Fund (ITF) MHP/009/20, and the Project of Hetao Shenzhen-Hong Kong Science and Technology Innovation Cooperation Zone under Grant HZQB-KCZYB-2020083. We would also like to thank the anonymous reviewers for their valuable feedback on the manuscript.

\bibliographystyle{icml2023}

\newpage
\appendix

\section{Experimental Settings and Additional Results}\label{appendix-exp}

We first provide a more in-depth description of experimental settings in Appendix~\ref{sec:exp-set}. Next, we present additional experimental results in Appendix~\ref{sec:add-exp}.

\subsection{Experimental Settings}\label{sec:exp-set}

To generate synthetic data, we first construct a graph structure randomly, then associate each edge with a weight that is uniformly sampled from $U(2,5)$. Then we obtain a weighted adjacency matrix $\bm A$, where $A_{ij}$ denotes the graph weight between node $i$ and node $j$. 

To generate an underlying precision matrix $\bm \Theta^\star \in \mathcal{M}^p$, we adopt the procedures used in \citet{slawski2015estimation,wang2020learning}. We set 
\begin{equation}\label{data_generate}
\bm \Theta = \delta \bm I - \bm A,  \quad \mathrm{with} \quad \delta = 1.05\lambda_{\max} \big(\bm A \big),
\end{equation}
where $\lambda_{\max} \big(\bm A \big)$ denotes the largest eigenvalue of $\bm A$. We set $\bm \Theta^\star = \bm E \bm \Theta \bm E$, where $\bm E$ is a diagonal matrix chosen such that the covariance matrix $(\bm \Theta^\star)^{-1}$ has unit diagonal elements. The matrix $\bm \Theta^\star$ as generated above must be an \textit{M}-matrix, but usually not diagonally dominant. Given a precision matrix $\bm \Theta^\star$, we generate $n$ independent and identically distributed observations $\bm x^{(1)}, \ldots, \bm x^{(n)} \sim \mathcal{N} (\bm 0, (\bm \Theta^\star)^{-1} )$. The sample covariance matrix is constructed by $\widehat{\bm \Sigma} = \frac{1}{n} \sum_{i=1}^n \bm x^{(i)} \bm (\bm x^{(i)})^\top$.

We also conduct experiments for estimating diagonally dominant \textit{M}-matrices. To generate an underlying precision matrix $\bm \Theta^\star \in \mathcal{M}_D^p$, we construct $\bm \Theta^\star = \bm D - \bm A + \bm V$, where $\bm D$ and $\bm V$ are diagonal matrices with each $D_{ii} = \sum_{j = 1}^p A_{ij}$ and each $V_{ii}$ uniformly sampled from $U(0,1)$.

To evaluate the estimation performance across different methods, we compute the estimation error as $\big \| \widehat{\bm \Theta}-\bm \Theta^{\star} \big \|_{\mathrm{F}}/\big \| \bm \Theta^{\star}\big \|_{\mathrm{F}}$, where $\widehat{\bm \Theta}$ and $\bm \Theta^{\star} $ denote the estimated and underlying precision matrices, respectively. The true positive rate (TPR) and false positive rate (FPR) are defined as
\begin{equation*}
\mathrm{TPR} = \frac{\mathrm{TP}}{\mathrm{TP} + \mathrm{FN}},  \quad \mathrm{and} \quad  \mathrm{FPR} = \frac{\mathrm{FP}}{\mathrm{FP} + \mathrm{TN}},
\end{equation*}
where $\mathrm{TP}$, $\mathrm{FP}$, $\mathrm{TN}$ and $\mathrm{FN}$ denote the number of true positives, false positives, true negatives, and false negatives, respectively. The true positive rate denotes the proportion of correctly identified edges in the graph, and the false positive rate denotes the proportion of incorrectly identified edges in the graph. An estimate with a higher true positive rate and a lower false positive rate indicates a better performance of identifying the underlying graph edges. The $\mathrm{F}$-$\mathrm{score}$ of a graph is defined as
\begin{equation}\label{measure}
 \mathrm{F}\text{-}\mathrm{score} := \frac{2\mathrm{TP}}{2\mathrm{TP}+\mathrm{FP}+\mathrm{FN}}.
\end{equation} 
The $\mathrm{F}\text{-}\mathrm{score}$ takes values in $[0, 1]$, and $1$ indicates perfect structure recovery. Our method uses the weight-updating function corresponding to the SCAD penalty.

\subsection{Additional Experimental Results}\label{sec:add-exp}

We perform experiments to demonstrate the empirical convergence of the proposed gradient projection algorithm, as illustrated in Figure \ref{fig:Convergence_illustration}. We observe that the number of iterations required for convergence descends rapidly with increasing stage $k$, and empirically the algorithm enjoys a linear convergence rate.

\begin{figure}[htbp]
\centerline{\includegraphics[scale=.45]{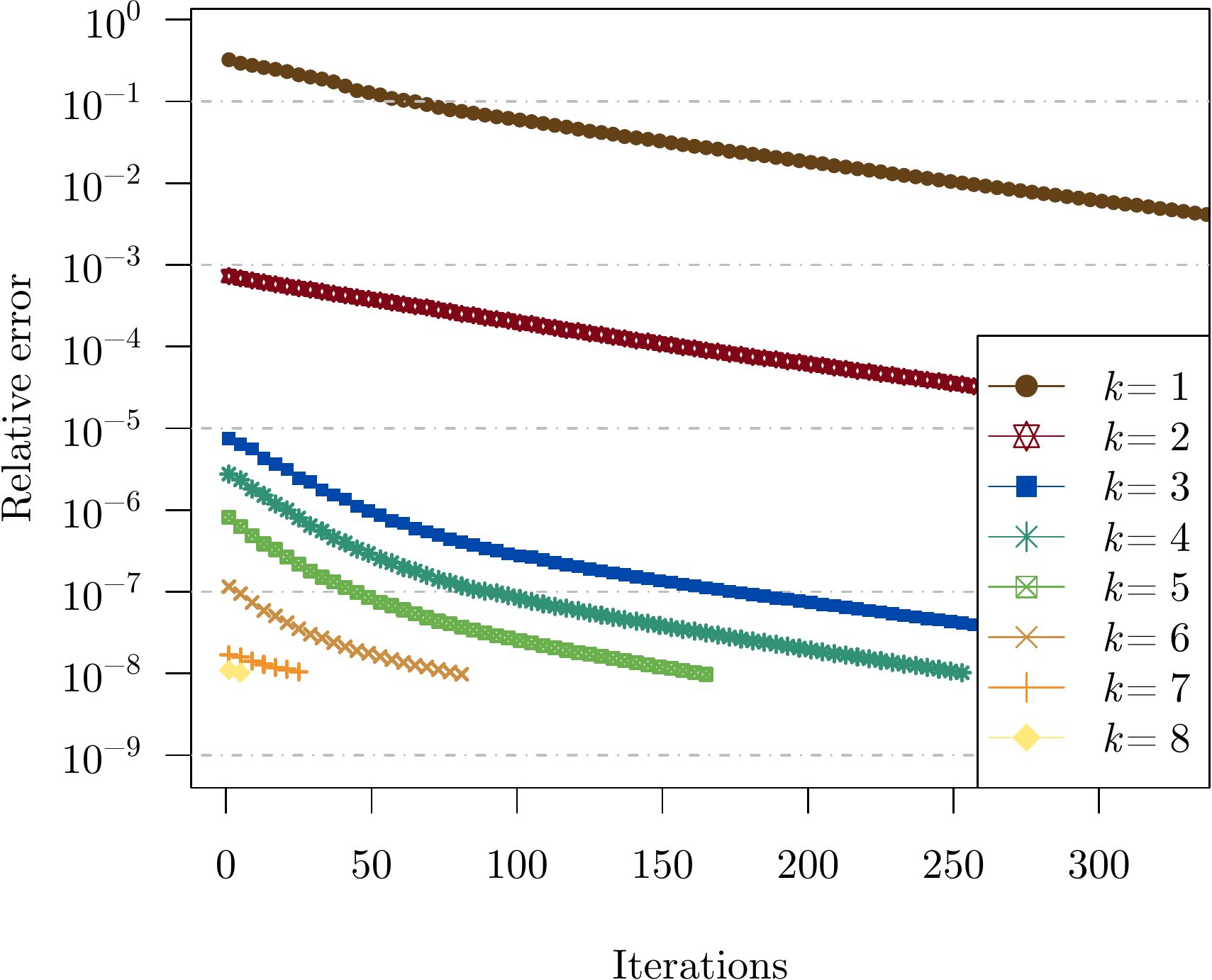}}
\caption{An empirical convergence result of the gradient projection algorithm at different stages. At the $k$-th stage with $k = 1, 2, \ldots, 8$, we plot the relative error $\big \| \bm \Theta_t - \widehat{\bm \Theta}^{(k)} \big \|_{\mathrm{F}}/\big \| \widehat{\bm \Theta}^{(k)}\big \|_{\mathrm{F}}$ as a function of the iteration $t$, where $\left\{\bm \Theta_t\right\}_{t \geq 1}$ is established by the gradient projection algorithm, and $\widehat{\bm \Theta}^{(k)}$ is the optimal solution at the $k$-th stage.}
\label{fig:Convergence_illustration}
\vspace{-0.2cm}
\end{figure}

When estimating diagonally dominant \textit{M}-matrices, our proposed algorithm involves double loops of iterations to solve the subproblem at each stage, potentially leading to computational inefficiency due to the use of Dykstra's algorithm for handling the diagonally dominant \textit{M}-matrix (DDM-matrix) constraint. Table~\ref{tab:methods_comparison} displays the running time of each algorithm for the case of estimating DDM-matrices. We observe that our algorithm is less computationally efficient in this scenario, primarily due to managing multiple constraints to ensure a DDM-matrix solution. To enhance our algorithm's efficiency, we could compute an approximate solution when addressing subproblems in earlier stages, a technique employed in previous studies \citep{xiao2013proximal,fan2018lamm}. Investigating the development of numerical algorithms for handling diagonally dominant $M$-matrix constraints more efficiently would be a valuable future research direction.

\begin{table}[h!]
\centering
\caption{Comparison of running times for methods in the case of estimating diagonally dominant \textit{M}-matrices.}
\vspace{-0.1cm}
\setlength{\tabcolsep}{12pt}
\begin{tabular}{@{}lcccccc@{}}
\toprule
Methods & \textsf{SLTP} & \textsf{GLASSO} & \textsf{CLIME} & \textsf{GSCAD} & \textsf{DDGL} & \textsf{Proposed} \\ \midrule
Running time (s) & 9.48 & 0.05 & 4.06 & 0.14 & 1.17 & 7.85 \\ \bottomrule
\end{tabular}
\label{tab:methods_comparison}
\end{table}

It is essential to emphasize that our proposed algorithm is the sole method effectively addressing the Gaussian maximum likelihood estimation problem under the DDM-matrix constraints. In contrast, \textsf{GLASSO}, \textsf{CLIME}, and \textsf{GSCAD} tackle unconstrained problems and cannot guarantee a DDM-matrix solution. Such a solution provides desirable spectral properties for performing the graph Fourier transform and is crucial for analyzing positive dependence in multivariate Pareto distributions. Furthermore, the popular Block Coordinate Descent (BCD) algorithm cannot ensure convergence to the global minimizer, as shown in Table~\ref{tab:algorithm_comparison}.

We compare the optimality gap for solutions obtained by the BCD algorithm and our proposed algorithm. The optimality gap is evaluated using two approaches: 
\begin{equation*}
\Delta_{\Theta} = \Vert \widehat{\bm \Theta} - \bm \Theta^* \Vert_{\mathrm{F}} / \Vert \bm \Theta^* \Vert_{\mathrm{F}}, \qquad \text{and} \qquad \Delta_f = \hat{f} - f^*,
\end{equation*}
where $\widehat{\bm \Theta}$ is the output from BCD or our algorithm, $\bm \Theta^*$ is the global minimizer obtained through CVX, and $\hat{f}$ and $f^*$ represent the objective function values at $\widehat{\bm \Theta}$ and $\bm \Theta^*$, respectively. Experiments are conducted on the Barabasi-Albert model with 50 nodes ($p = 50$) and the number of samples $n$ ranging from 50 to 5000. For both algorithms, the stopping criterion is met when successive iterations satisfy the condition $\Vert \bm \Theta_{t+1} - \bm \Theta_{t} \Vert_{\mathrm{F}} / \Vert \bm \Theta_t \Vert_{\mathrm{F}} < 10^{-12}$.

Table~\ref{tab:algorithm_comparison} reveals that the popular BCD algorithm exhibits a substantial optimality gap, indicating its failure to converge to the global minimizer. In contrast, our proposed algorithm effectively converges to the minimizer. We note that the BCD algorithm can perform well under the \textit{M}-matrix constraint without diagonal dominance.

\begin{table}[h!]
\centering
\caption{Optimality gap comparison for algorithms in the case of estimating diagonally dominant \textit{M}-matrices.}
\vspace{-0.1cm}
\setlength{\tabcolsep}{22pt}
\begin{tabular}{@{}lcc@{}}
\toprule
Algorithms & BCD & Proposed \\ \midrule
$n=50$      & $\Delta_{\Theta} = 0.5295, \quad \Delta_f = 0.0738 $ & $\Delta_{\Theta} = 2.81 \times 10^{-6}, \quad \Delta_f = 1.49 \times 10^{-12}$ \\
$n=500$     & $ \Delta_{\Theta} = 0.0765, \quad \Delta_f = 0.0022$ & $\Delta_{\Theta} = 1.09 \times 10^{-5}, \quad \Delta_f = 2.68 \times 10^{-12}$ \\
$n=5000$    & $\Delta_{\Theta} = 0.0166, \quad \Delta_f = 0.0004$ & $ \Delta_{\Theta} = 1.39 \times 10^{-5}, \quad \Delta_f = 2.77 \times 10^{-12}$ \\ \bottomrule
\end{tabular}
\label{tab:algorithm_comparison}
\end{table}

\section{Proofs}\label{app-proof}

We first present two lemmas in Appendix~\ref{lemmas}. Then we provide the proofs of Theorem\ref{converge-pgd} in Appendix~\ref{sec:proof-pgd}, Proposition~\ref{solution-PB} in Appendix~\ref{appendix-pb}, and Theorem~\ref{theorem 2} and Corollary~\ref{corollary-result} in Appendix~\ref{proof-sec}.

We begin by introducing the notations used in proofs. Define
\begin{equation}
L_\alpha (\bm \Theta) := \big\lbrace (i, j) \, \big| \, | \Theta_{ij} | \geq  \alpha, \, i \neq j \big\rbrace,
\end{equation}
where $\alpha = c_0 \lambda$ with $c_0$ defined in Assumption \ref{assumption 1}. Then we define the set $T_\alpha (\bm \Theta)$ as follows:
\begin{equation}\label{T-set}
T_\alpha (\bm \Theta) := L_\alpha (\bm \Theta) \cup S^\star,
\end{equation}
where $S^\star$ is defined in \eqref{eq:set-S}. Define
\begin{equation*}
G_{\lambda, \widehat{\bm \Sigma}} (\widetilde{\bm \Theta}) := \underset{\bm \Theta \in \mathcal{M}^p}{\mathsf{arg \ min}} - \log \det(\bm \Theta) + \trace \big(\bm \Theta \widehat{\bm \Sigma} \big) + \sum_{i \neq j} p_\lambda ( | \widetilde{\Theta}_{ij} |)   \left| \Theta_{ij} \right|.
\end{equation*}
For ease of presentation, by a slight abuse of notation, we write $p_\lambda \big( | \bm \Theta_{S^\star}|\big) $ for $\big(  p_\lambda ( |\Theta_{ij}| ) \big)_{(i,j) \in {S^\star}}$.

\subsection{Lemmas}\label{lemmas}

\begin{lemma}\label{lem2}
Let $g(\bm \Theta) = -\log \det (\bm \Theta)$. Define a local region of $\bm \Theta^{\star}$ by
\begin{equation*}
\mathcal{B} \left(\bm \Theta^\star, r \right) = \left\lbrace \bm \Theta \in \mathcal{M}^p \big| \norm{\bm \Theta - \bm \Theta^\star}_{\mathrm{F}} \leq  r \right\rbrace.
\end{equation*}
Then for any $\bm \Theta_1, \bm \Theta_2 \in \mathcal{B} \left(\bm \Theta^{\star}; \lambda_{\max} \left( \bm \Theta^{\star} \right) \right)$, we have
\begin{equation*}
\left\langle \nabla g \left(\bm \Theta_1 \right) - \nabla g \left(\bm \Theta_2 \right), \bm \Theta_1 - \bm \Theta_2 \right\rangle \geq  \frac{1}{4}\lambda_{\max}^{-2} \left( \bm \Theta^{\star} \right) \norm{\bm \Theta_1 - \bm \Theta_2}_{\mathrm{F}}^2.
\end{equation*}
\end{lemma}

\begin{proof}
For any $\bm \Theta_1$, $\bm \Theta_2 \in \mathcal{B} \left(\bm \Theta^{\star}; \lambda_{\max} \left( \bm \Theta^{\star} \right) \right)$, by Mean Value Theorem, one obtains
\begin{equation*}
g \left(\bm \Theta_2 \right) = g \left(\bm \Theta_1 \right) + \left\langle\nabla g \left(\bm \Theta_1 \right), \, \bm \Theta_2 - \bm \Theta_1 \right\rangle  + \frac{1}{2} \vect{ \bm \Theta_2 - \bm \Theta_1}^\top \nabla^2 g \left(\bm \Theta_t \right)  \vect{\bm \Theta_2 -\bm \Theta_1}, 
\end{equation*}
where $\bm \Theta_t = t\bm \Theta_2 + (1-t)\bm \Theta_1$ with $t \in [0, 1]$. Note that $\bm \Theta_t \in \mathcal{B} \left(\bm \Theta^{\star};r \right) $. One has
\begin{equation*}
\lambda_{\min} \left( \nabla^2 g \left(\bm \Theta_t \right) \right) = \lambda_{\min} \big(  \left( \bm \Theta_t \otimes \bm \Theta_t \right)^{-1} \big) = \lambda_{\max}^{-2} \left( \bm \Theta_t  \right), 
\end{equation*}
where the second equality follows from the property that the eigenvalues of $\bm A \otimes \bm B$ are $\lambda_i \mu_j$, where $\lambda_1, \ldots, \lambda_p$ and $\mu_1, \ldots, \mu_p$ are the eigenvalues of $\bm A \in \mathbb{R}^{p \times p}$ and $\bm B \in \mathbb{R}^{p \times p}$, respectively. Following from the Weyl's inequality, one obtains
\begin{equation*}
\lambda_{\max} \left( \bm \Theta_t  \right) \leq  \lambda_{\max} \left( \bm \Theta_t - \bm \Theta^\star \right) + \lambda_{\max} \left(\bm \Theta^\star \right)  \leq  r + \lambda_{\max} \left( \bm \Theta^\star \right).
\end{equation*} 
Therefore, We have
\begin{equation}\label{nq24}
g \left(\bm \Theta_2 \right) \geq g \left(\bm \Theta_1 \right) + \left\langle\nabla g \left(\bm \Theta_1 \right), \bm \Theta_2 - \bm \Theta_1 \right\rangle  + \frac{1}{2} \left( r + \lambda_{\max} \left( \bm \Theta^\star \right) \right)^{-2}\norm{ \bm \Theta_2 - \bm \Theta_1}_{\mathrm{F}}^2,
\end{equation}
and
\begin{equation}\label{nq25}
g \left(\bm \Theta_1 \right) \geq g \left(\bm \Theta_2 \right) + \left\langle\nabla g \left(\bm \Theta_2 \right), \bm \Theta_1 - \bm \Theta_2\right\rangle + \frac{1}{2} \left( r + \lambda_{\max} \left( \bm \Theta^\star \right) \right)^{-2}\norm{ \bm \Theta_2 - \bm \Theta_1}_{\mathrm{F}}^2.
\end{equation}
Setting $r=\lambda_{\max} \left( \bm \Theta^\star \right)$ and combining \eqref{nq24} and \eqref{nq25}, we obtain
\begin{equation*}
\left\langle \nabla g \left(\bm \Theta_1 \right) - \nabla g \left(\bm \Theta_2 \right), \bm \Theta_1 - \bm \Theta_2 \right\rangle \geq  \frac{1}{4}\lambda_{\max}^{-2} \left( \bm \Theta^{\star} \right) \norm{\bm \Theta_1 - \bm \Theta_2}_{\mathrm{F}}^2.
\end{equation*}
completing the proof.
\end{proof}

\begin{lemma}\label{lem3}
Suppose the event $\big\|  \bm \Sigma^{\star}  - \widehat{\bm \Sigma} \big\|_{\max} \leq  \frac{\sqrt{2}}{2}\lambda$ holds and the sample size satisfies $n \geq 8 c_0 c_1 \tau   s \log p$, and take $\lambda = \sqrt{ c_1 \tau \log p /n}$.
Under Assumptions \ref{assumption 1} and \ref{assumption 2}, if $| T_{\alpha} (\widetilde{\bm \Theta})| \leq 2s-p$ holds for some $\widetilde{\bm \Theta}$, then 
\begin{equation*}
\big\|G_{\lambda, \widehat{\bm \Sigma}} (\widetilde{\bm \Theta}) - \bm \Theta^{\star} \big\|_{\mathrm{F}} \leq \frac{c_0}{2} \big( \big\| p_\lambda \big( \big| \widetilde{\bm \Theta}_{S^\star} \big| \big) \big\| + \big\| \big(  \bm \Sigma^{\star} - \widehat{\bm \Sigma}  \big)_{T_{\alpha} (\widetilde{\bm \Theta}) \cup I_p} \big\| \big),
\end{equation*}
and
\begin{equation}
\big | T_\alpha (G_{\lambda, \widehat{\bm \Sigma}} (\widetilde{\bm \Theta})) \big| \leq 2s-p.
\end{equation}
\end{lemma}

\begin{proof}
Drawing inspiration from \citet{sun2018graphical}, our proof employs techniques that establish upper-bounded estimation error in regularized maximum likelihood estimation problems using Karush–Kuhn–Tucker (KKT) conditions. Notably, our formulation incorporates multiple constraints, contrasting the unconstrained problem in \citet{sun2018graphical}, which leads to distinct KKT conditions and error-bounding techniques.

Due to the sign constraint, $G_{\lambda, \widehat{\bm \Sigma}} (\widetilde{\bm \Theta})$ can be expressed as the minimizer of the following problem:
\begin{equation}\label{Theta_1}
\underset{{\bm \Theta \in \mathcal{M}^p}}{\mathsf{minimize}}  - \log \det (\bm \Theta) + \trace \big(\bm \Theta \widehat{\bm \Sigma} \big)  - \sum_{i \neq j} W_{ij}  \Theta_{ij},
\end{equation}
where $W_{ij} = p_\lambda \big( \big|\widetilde{\Theta}_{ij} \big| \big)$ if $i \neq j$, and zero otherwise. Define a local region 
\begin{equation*}
\mathcal{B} \left( \bm \Theta^\star, \lambda_{\max} \left( \bm \Theta^{\star} \right)  \right)=\{\bm \Theta \in \mathcal{M}^p \, | \norm{ \bm \Theta - \bm \Theta^{\star}}_{\mathrm{F}} \leq \lambda_{\max}( \bm \Theta^{\star} \}.
\end{equation*}

For ease of presentation, we denote $G_{\lambda, \widehat{\bm \Sigma}} (\widetilde{\bm \Theta})$ by $\widehat{\bm \Theta}$. We first prove that $\widehat{\bm \Theta} \in \mathcal{B} \left( \bm \Theta^\star, \lambda_{\max} \left(\bm \Theta^{\star} \right) \right)$ under event $\mathcal{J}$ defined in \eqref{T1-event-J}. We construct an intermediate estimate:
\begin{equation}\label{wt}
\bm \Theta_t := \bm \Theta^{\star} + t \big(\widehat{\bm \Theta} - \bm \Theta^{\star} \big), 
\end{equation}
where $t$ is taken such that $\norm{\bm \Theta_t - \bm \Theta^{\star}}_{\mathrm{F}} = \lambda_{\max}(\bm \Theta^{\star})$ if $\big\|\widehat{\bm \Theta} - \bm \Theta^{\star} \big\|_{\mathrm{F}} > \lambda_{\max} \left(\bm \Theta^{\star} \right)$, and $t=1$ otherwise. Hence $\norm{\bm \Theta_t - \bm \Theta^{\star}}_{\mathrm{F}} \leq \lambda_{\max}(\bm \Theta^{\star})$ holds for any $t \in [0, 1]$. Then we apply Lemma \ref{lem2} with $\bm \Theta_1 = \bm \Theta_t$, $\bm \Theta_2 = \bm \Theta^\star$, and get
\begin{equation}\label{T1-q1}
\frac{c_0}{2}  t \big\langle ( \bm \Theta^\star )^{-1} -  \bm \Theta_t^{-1}, \, \widehat{\bm \Theta} - \bm \Theta^\star \big\rangle  \geq  \norm{ \bm \Theta_t - \bm \Theta^\star}_{\mathrm{F}}^2. 
\end{equation}
Construct a function $q(t) := -\log \det \big(\bm \Theta^\star + t(\widehat{\bm \Theta}- \bm \Theta^\star)  \big) + t\langle(\bm \Theta^\star)^{-1}, 
\widehat{\bm \Theta} - \bm \Theta^\star \rangle$ and $t \in [0, 1]$. One has
\begin{equation}\label{T1-q2}
q'(t)= \big\langle \left( \bm \Theta^\star \right)^{-1}-  \bm \Theta_t^{-1}, \ \widehat{\bm \Theta} - \bm \Theta^\star \big\rangle, 
\end{equation}
Let $\bm A = \bm \Theta_t^{-1}$, and $\bm B = \widehat{\bm \Theta} - \bm \Theta^\star$. Then one also has
\begin{equation*}
q''(t) = \tr{\bm A \bm B \bm A \bm B},   
\end{equation*}
Next, we show that $q''(t) \geq 0$. Let $\bm C = \bm A \bm B$. Following from \citet[Theorem~1]{drazin1962criteria}, all eigenvalues of a matrix $\bm X$ are real if there exists a symmetric and positive definite matrix $\bm Y$ such that $\bm X \bm Y$ are symmetric. Since $\bm C \bm A$ is symmetric and $\bm A$ symmetric and positive definite, all eigenvalues of $\bm C$ are real. Suppose $\lambda_1, \ldots, \lambda_p$ are the eigenvalues of $\bm C$. Then one obtains $q''(t) = \sum_{i=1}^p \lambda_i^2 \geq 0$, implying that $q'(t)$ is monotonically non-decreasing. Then one obtains
\begin{equation} \label{T1-d4}
\big\langle  \left( \bm \Theta^\star \right)^{-1} - \widehat{\bm \Theta}^{-1}, \widehat{\bm \Theta} - \bm \Theta^\star \big\rangle \geq q'(t) \geq \frac{2}{c_0 t}\norm{ \bm \Theta_t - \bm \Theta^\star}_{\mathrm{F}}^2.
\end{equation}
where the first inequality follows from $q'(1)\geq q'(t)$ and $t \leq 1$, and the second inequality follows from \eqref{T1-q1}.

Following from \eqref{T1-d4}, we have 
\begin{equation} \label{d4}
\norm{ \bm \Theta_t - \bm \Theta^\star}_{\mathrm{F}}^2 \leq \frac{c_0}{2}  t \big\langle  \left( \bm \Theta^\star \right)^{-1} -  \widehat{\bm \Theta}^{-1}, \widehat{\bm \Theta} - \bm \Theta^\star \big\rangle. 
\end{equation}
We present the Lagrangian of the optimization \eqref{Theta_1} as follows:
\begin{equation*}
L \left(\bm \Theta, \bm \Gamma \right) = - \log \det (\bm \Theta) + \trace \big(\bm \Theta \widehat{\bm \Sigma} \big)  - \sum_{i \neq j} W_{ij}  \Theta_{ij} + \langle \bm \Gamma,  \bm \Theta \rangle,
\end{equation*}
where $\bm \Gamma$ is a KKT multiplier with $\Gamma_{ii} =0$ for $i \in [p]$. Let $ \big(\widehat{\bm \Theta}, \widehat{\bm \Gamma} \big)$ be the primal and dual optimal point, which must satisfy the KKT conditions as follows:
\begin{align}
- \widehat{\bm \Theta}^{-1} + \widehat{\bm \Sigma} - \bm W + \widehat{\bm \Gamma} = \bm 0&; \label{kkt1}\\
\widehat{\Theta}_{ij} \widehat{\Gamma}_{ij} =0, \ \widehat{\Theta}_{ij} \leq 0, \ \widehat{\Gamma}_{ij} \geq 0, \ \forall \ i\neq j  &; \label{kkt2}\\
\widehat{\Gamma}_{ii}=0, \ \mathrm{for} \ i \in [p] &; \label{kkt3}
\end{align}
According to \eqref{kkt1}, one has
\begin{equation}
 \big\langle - \widehat{\bm \Theta}^{-1} + \widehat{\bm \Sigma}, \, \widehat{\bm \Theta} - \bm \Theta^{\star} \big\rangle = \big\langle \bm W -\widehat{ \bm \Gamma}, \widehat{\bm \Theta} - \bm \Theta^{\star} \big\rangle.     \label{kkt1-new}
\end{equation}
Plugging \eqref{kkt1-new} into \eqref{d4} yields
\begin{equation}\label{temrs}
\norm{\bm \Theta_t - \bm \Theta^{\star}}_{\mathrm{F}}^2 = \frac{c_0}{2}  t \Big( \underbrace{ \big\langle \widehat{\bm \Gamma}, \bm \Theta^{\star} - \widehat{\bm \Theta}  \big\rangle }_{\mathrm{term \, \uppercase\expandafter{\romannumeral1}}}  +  \underbrace{ \big\langle \bm W, \widehat{\bm \Theta} - \bm \Theta^{\star} \big\rangle}_{\mathrm{term \, \uppercase\expandafter{\romannumeral2}}}  + \underbrace{\big\langle  \bm \Sigma^\star - \widehat{\bm \Sigma}, \widehat{\bm \Theta} - \bm \Theta^{\star} \big\rangle}_{\mathrm{term \, \uppercase\expandafter{\romannumeral3}}} \Big). 
\end{equation}

The term \uppercase\expandafter{\romannumeral1} in \eqref{temrs} can be bounded by
\begin{equation*}
 \big\langle \widehat{\bm \Gamma}, \bm \Theta^{\star} - \widehat{\bm \Theta}  \big\rangle = \sum_{i \neq j} \widehat{\Gamma}_{ij} \Theta^{\star}_{ij} \leq 0.
\end{equation*}
where the equality follows from \eqref{kkt2} and \eqref{kkt3}, and the inequality follows from $\widehat{\Gamma}_{ij} \geq 0$ in \eqref{kkt2} and the fact that $\Theta^\star_{ij} \leq 0$ for any $i \neq j$.

To bound term \uppercase\expandafter{\romannumeral2}, we separate the support of $\bm W$ off the diagonal into two parts, $S^\star$ and $T^\star$ which is defined as follows:
\begin{equation*}
T^\star = \left\lbrace ( i, j )  \, | \, \Theta^{\star}_{ij} = 0, i \neq j \right\rbrace.
\end{equation*}
Let $I_p := \left\lbrace (i, i) \, | \, i \in [p] \right\rbrace$. Define
\begin{equation}
\bar{T}_\alpha (\widetilde{\bm \Theta}) = T_\alpha (\widetilde{\bm \Theta}) \cup I_p \quad \mathrm{and} \quad \bar{T}_\alpha^c (\widetilde{\bm \Theta}) = T_\alpha^c (\widetilde{\bm \Theta}) \setminus I_p.
\end{equation}
Since $|T_\alpha (\widetilde{\bm \Theta})|\leq 2s-p$, $|\bar{T}_\alpha (\widetilde{\bm \Theta})|\leq 2s$. Then one has
\begin{equation*}
\begin{split}
\big\langle \bm W, \, \widehat{\bm \Theta} - \bm \Theta^{\star} \big\rangle
&= \big\langle \bm W_{S^\star}, \, \big( \widehat{\bm \Theta} - \bm \Theta^{\star} \big)_{S^\star} \big\rangle + \big\langle \bm W_{T^\star}, \, \widehat{\bm \Theta}_{T^\star} \big\rangle \\
& \leq  \norm{ \bm W_{S^\star} } \big\| \big( \widehat{\bm \Theta} - \bm \Theta^{\star} \big)_{S^\star} \big\| + \big\langle \bm W_{ \bar{T}_\alpha^c (\widetilde{\bm \Theta})}, \, \widehat{\bm \Theta}_{ \bar{T}_\alpha^c (\widetilde{\bm \Theta})} \big\rangle,    
\end{split}
\end{equation*}
where the inequality follows from the fact that $ W_{ij} \geq 0$, $\widehat{\Theta}_{ij} \leq 0$ for $i \neq j$, and $\bar{T}_\alpha^c (\widetilde{\bm \Theta}) \subseteq T^\star$.

For term \uppercase\expandafter{\romannumeral3}, we separate the support of $\big(\bm \Sigma^\star - \widehat{\bm \Sigma} \big)$ into parts, $\bar{T}_\alpha( \widetilde{\bm \Theta})$ and $\bar{T}_\alpha^c(\widetilde{\bm \Theta})$. Then one has
\begin{equation*}\label{t3}
\big\langle  \bm \Sigma^\star - \widehat{\bm \Sigma}, \, \widehat{\bm \Theta} - \bm \Theta^{\star} \big\rangle
\leq  \big\langle \big( \bm \Sigma^\star - \widehat{\bm \Sigma} \big)_{\bar{T}_\alpha^c(\widetilde{\bm \Theta})}, \widehat{\bm \Theta}_{\bar{T}_\alpha^c(\widetilde{\bm \Theta})} \big\rangle  + \big\| \big(  \bm \Sigma^\star - \widehat{\bm \Sigma} \big)_{\bar{T}_\alpha (\widetilde{\bm \Theta})} \big\| \big\| \big( \widehat{\bm \Theta} - \bm \Theta^{\star}  \big)_{\bar{T}_\alpha (\widetilde{\bm \Theta})} \big\|.
\end{equation*}
We note that, for any $(i, j) \in \bar{T}_\alpha^c (\widetilde{\bm \Theta})$, $\widehat{\Theta}_{ij} \leq 0$ and $W_{ij} \geq \frac{\sqrt{2}}{2}\lambda$, following from the definition of $T_\alpha$ in \eqref{T-set} and Assumption~\ref{assumption 1}.  Then under the event that $\big\|  \bm \Sigma^{\star}  - \widehat{\bm \Sigma} \big\|_{\max} \leq  \frac{\sqrt{2}}{2}\lambda$, we have
\begin{equation}\label{inq-lem}
\big\langle \big( \bm \Sigma^\star - \widehat{\bm \Sigma} + \bm W \big)_{\bar{T}_\alpha^c (\widetilde{\bm \Theta})}, \, \widehat{\bm \Theta}_{\bar{T}_\alpha^c (\widetilde{\bm \Theta})} \big\rangle \leq 0.
\end{equation}

By bounding term \uppercase\expandafter{\romannumeral1}, term \uppercase\expandafter{\romannumeral2}, and term \uppercase\expandafter{\romannumeral3} in \eqref{temrs}, together with \eqref{inq-lem}, we obtain
\begin{equation}\label{bound_theta}
\big\| \bm \Theta_t - \bm \Theta^{\star} \big\|_{\mathrm{F}}^2 \leq \frac{c_0}{2}  t \Big( \big\| \bm W_{S^\star} \big\| \big\| \big( \widehat{\bm \Theta} - \bm \Theta^{\star} \big)_{S^\star} \big\|  + \big\| \big( \bm \Sigma^\star - \widehat{\bm \Sigma} \big)_{\bar{T}_\alpha (\widetilde{\bm \Theta})} \big\| \big\| \big( \widehat{\bm \Theta} - \bm \Theta^{\star}  \big)_{\bar{T}_\alpha (\widetilde{\bm \Theta})} \big\| \Big).
\end{equation}

Through the definition of $\bm \Theta_t$ in \eqref{wt}, one has $\norm{ \bm \Theta_t - \bm \Theta^{\star}}_{\mathrm{F}} = t \big\| \widehat{\bm \Theta} - \bm \Theta^{\star} \big\|_{\mathrm{F}}$. Then \eqref{bound_theta} can be written as follows:
\begin{equation}\label{bound_theta_1}
\norm{ \bm \Theta_t - \bm \Theta^{\star}}_{\mathrm{F}} \leq \frac{c_0}{2}  \Big( \big\| \bm W_{S^\star} \big\| + \big\| \big(  \bm \Sigma^\star - \widehat{\bm \Sigma} \big)_{\bar{T}_\alpha (\widetilde{\bm \Theta}) } \big\| \Big).  
\end{equation}

Recall that $W_{ij} = p_\lambda \big( \big |\widetilde{\Theta}_{ij} \big|\big) \leq \lambda$ for any $i \neq j$ according to Assumption \ref{assumption 1}, and $|S^\star| \leq s - p$. Thus one has
\begin{equation*}
\norm{\bm W_{S^\star}}\leq \sqrt{s} \lambda. 
\end{equation*}
Under the event $\mathcal{J}$, one can bound
\begin{equation}\label{bound_sigma_s}
\big\| \big( \bm \Sigma^\star - \widehat{\bm \Sigma} \big)_{ \bar{T}_\alpha (\widetilde{\bm \Theta})} \big\| \leq |\bar{T}_\alpha (\widetilde{\bm \Theta})|^{\frac{1}{2}}  \big\| \big(  \bm \Sigma^\star   - \widehat{\bm \Sigma} \big)_{ \bar{T}_\alpha (\widetilde{\bm \Theta})} \big\|_{\max} \leq \sqrt{s} \lambda, 
\end{equation}
where the second inequality follows from $|\bar{T}_\alpha (\widetilde{\bm \Theta})|\leq 2s$.

Plugging \eqref{bound_theta_1} and \eqref{bound_sigma_s} into \eqref{bound_theta} yields
\begin{equation*}
\norm{ \bm \Theta_t - \bm \Theta^{\star}}_{\mathrm{F}}  \leq  c_0 \sqrt{s} \lambda \leq \lambda_{\max} \left( \bm \Theta^{\star} \right), 
\end{equation*}
which indicates that $t=1$ in \eqref{wt}, \textit{i.e.}, $\bm \Theta_t = \widehat{\bm \Theta}$. The last inequality is established by plugging $\lambda = \sqrt{ c_1 \tau \log p /n}$ with $n \geq 8  c_0 c_1 \tau   s \log p$.
Therefore, we obtain
\begin{equation}\label{bound-hat-theta}
\big\| \widehat{\bm \Theta} - \bm \Theta^{\star} \big\|_{\mathrm{F}} \leq c_0 \sqrt{s} \lambda.
\end{equation}
Putting $\bm W_{S^\star} = p_\lambda ( | \widetilde{\bm \Theta}_{S^\star}|)$ and $\bm \Theta_t=G_{\lambda, \widehat{\bm \Sigma}} (\widetilde{\bm \Theta})$ into \eqref{bound_theta_1} gets
\begin{equation*}
\big\| G_{\lambda, \widehat{\bm \Sigma}} (\widetilde{\bm \Theta}) - \bm \Theta^{\star} \big\|_{\mathrm{F}} \leq \frac{c_0}{2}  \big( \big\| p_\lambda \big( \big| \widetilde{\bm \Theta}_{S^\star} \big| \big) \big\| + \big\| \big(  \bm \Sigma^\star - \widehat{\bm \Sigma} \big)_{\bar{T}_\alpha (\widetilde{\bm \Theta}) } \big\| \big).  
\end{equation*}

To establish $\big | T_\alpha (\widehat{\bm \Theta}) \big| \leq 2s-p$, we separate the set $T_\alpha (\widehat{\bm \Theta})$ into two parts, $S^\star$ and $L_\alpha (\widehat{\bm \Theta}) \backslash S^\star$. For any $(i, j) \in L_\alpha (\widehat{\bm \Theta}) \backslash S^\star$, one has $\big|\widehat{ \Theta }_{ij}\big| \geq \alpha$, and further obtains
\begin{equation*}
\big|L_\alpha (\widehat{\bm \Theta}) \backslash S^\star \big|  \leq  \frac{\big\| \widehat{\bm \Theta}_{ L_\alpha (\widehat{\bm \Theta}) \backslash S^\star} \big\|^2 }{\alpha^2 }  \leq \frac{\big\| \widehat{\bm \Theta} - \bm \Theta^{\star} \big\|_{\mathrm{F}}^2}{\alpha^2} \leq s,
\end{equation*}
where the last inequality follows from \eqref{bound-hat-theta}. Then we have
\begin{equation*}
\big | T_\alpha (\widehat{\bm \Theta}) \big| = \big|S^\star \cup L_\alpha (\widehat{\bm \Theta}) \backslash S^\star \big| = \big|S^\star \big|+ \big| L_\alpha (\widehat{\bm \Theta}) \backslash S^\star \big| \leq 2s - p,
\end{equation*}
completing the proof.
\end{proof}

\subsection{Proof of Theorem \ref{converge-pgd}}\label{sec:proof-pgd}
\begin{proof}
Our convergence analysis largely builds upon standard techniques for the projected gradient descent, as provided in \citet{beck2014introduction}. The main difference lies in the open constraint set in our formulation, which leads to an additional requirement regarding the positive definiteness of the matrix during stepsize selection, extending beyond the backtracking condition. Consequently, it is crucial to incorporate certain adaptations in comparison to the standard methods employed for analyzing the convergence of projected gradient descent.

Given $\bm \Theta^o \in \mathcal{S}$, define the lower level set of the objective $f$ of Problem \eqref{eq:problem-S} as follows:
\begin{equation}\label{Level set}
L_f : = \left\lbrace \bm \Theta \in \mathcal{S} \, | \, f(\bm \Theta) \leq f(\bm \Theta^o) \right\rbrace.
\end{equation}

Following from \citep[Lemma~3.5]{cai2021fast}, there exists $m>0$ such that $\bm \Theta \succeq m \bm I$ holds for any $\bm \Theta \in L_f$. The gradient of $f$ is Lipschitz continuous with parameter $L = m^{-2}$ over $L_f$, since the Hessian is $\bm \Theta^{-1} \otimes \bm \Theta^{-1}$. 

Suppose that the initial point of the sequence $\bm \Theta_0 \in L_f$. To guarantee this, we may consider a $\bm \Theta^o$ with $f(\bm \Theta^o)$ sufficiently large. Let $\bm \Theta_t (\eta) = \mathcal{P}_{\mathcal{S}_d} \big( \bm \Theta_{t} - \eta \nabla f ( {\bm \Theta}_{t} ) \big)$. At the $t$-th iteration, if $\bm \Theta_t \in L_f$, then there must exist $\gamma >0$ such that for any $\eta \in (0, \gamma)$, $\bm \Theta_t (\eta)$ satisfies that $\bm \Theta_t (\eta) \in \mathbb{S}_{++}^p$. This is because $\bm \Theta_t$ is an interior point of $\mathbb{S}_{++}^p$, indicating that there exists $r>0$, such that for any $\bm \Theta$ satisfying $\|\bm \Theta - \bm \Theta_t\|_{\mathrm{F}} < r$, $\bm \Theta \in \mathbb{S}_{++}^p$. Taking $\eta < \frac{r}{\|  \nabla f ( {\bm \Theta}_{t} ) \|_{\mathrm{F}}}$, one has,
\begin{equation*}
\|\bm \Theta_t(\eta) - \bm \Theta_t \|_{\mathrm{F}} = \| \mathcal{P}_{\mathcal{S}_d} \big( \bm \Theta_{t} - \eta \nabla f ( {\bm \Theta}_{t} ) \big) - \mathcal{P}_{\mathcal{S}_d} \big( \bm \Theta_t\big) \|_{\mathrm{F}} \leq \eta \|  \nabla f ( {\bm \Theta}_{t} ) \|_{\mathrm{F}} < r,
\end{equation*}
establishing that $\bm \Theta_t(\eta) \in \mathbb{S}_{++}^p$. The projection $\mathcal{P}_{\mathcal{S}_d}$ ensures $\bm \Theta_t(\eta) \in \mathcal{S}_d$. Totally, $\bm \Theta_t(\eta) \in \mathcal{S}$.

The line search condition~\eqref{eq:line-search} ensures that $f(\bm \Theta_t (\eta)) \leq f(\bm \Theta_t)$, implying that $\bm \Theta_t (\eta) \in L_f$. Then one has
\begin{equation}\label{eq:descent-lemma}
f(\bm \Theta_t (\eta)) \leq f(\bm \Theta_t) + \left\langle \nabla f(\bm \Theta_t), \, \bm \Theta_t (\eta) - \bm \Theta_t\right\rangle  + \frac{L}{2} \| \bm \Theta_t (\eta) - \bm \Theta_t  \|_{\mathrm{F}}^2.
\end{equation} 
Following from projection theorem \citep[Theorem~9.8]{beck2014introduction}, one has
\begin{equation}
\left\langle \bm \Theta_t - \eta \nabla f(\bm \Theta_t) - \bm \Theta_t (\eta), \, \bm \Theta_t - \bm \Theta_t (\eta) \right\rangle.
\end{equation}
Together with \eqref{eq:descent-lemma}, one has
\begin{equation}
f(\bm \Theta_t (\eta)) \leq f(\bm \Theta_t) - \eta\Big(1 - \frac{\eta L}{2} \Big) \big\| G_{\frac{1}{\eta}} (\bm \Theta_t ) \big\|_{\mathrm{F}}^2.
\end{equation}
Taking $\eta \leq \frac{2(1-\alpha)}{L}$ leads to $1 - \frac{\eta L}{2} \geq \alpha$. Therefore, for any $\eta \leq \min \big(\frac{2(1-\alpha)}{L}, \gamma \big)$, $\bm \Theta_t (\eta)$ can simultaneously satisfy $\bm \Theta_t (\eta) \in \mathcal{S}$ and the line search condition \eqref{eq:line-search}. Thus, the step size in line search has a lower bound $\eta_t \geq \min \big(\frac{2(1-\alpha)\beta}{L}, \gamma, \sigma\big)$, and $\bm \Theta_{t+1} \in L_f$. By induction, $\bm \Theta_t \in L_f$ for any $t \geq 0$. Sequence $\left\lbrace f(\bm \Theta_t)\right\rbrace$ is monotonically decreasing, and $f(\bm \Theta_{t+1}) < f(\bm \Theta_t)$ until $G_{\frac{1}{\eta_t}} (\bm \Theta_t ) = \bm 0$, indicating that $\bm \Theta_t$ is a stationary point \citep[Theorem~9.10]{beck2014introduction}. Since Problem~\eqref{eq:problem-S} is a strictly convex problem, the stationary point is the unique minimizer. Therefore, the proposed algorithm converges to the optimal solution.
\end{proof}

\subsection{Proof of Proposition \ref{solution-PB}}\label{appendix-pb}
\begin{proof}
The Lagrangian of the optimization \eqref{P_B_row} is
\begin{equation*}
L (\bm x, \bm \mu) = \frac{1}{2}\| \bm x - \bm y \|^2 - \mu_r \bm x^\top \bm 1 + \langle \bm \mu_{\setminus r},\, \bm x_{\setminus r}\rangle,
\end{equation*}
where $\bm \mu$ is a KKT multiplier. Let $(\hat{\bm x}, \hat{\bm \mu} )$ be the primal and dual optimal point. Then $ (\hat{\bm x}, \hat{\bm \mu} )$ must satisfy the KKT system:
\begin{align}
\hat{x}_i - y_i - \hat{\mu}_r + \hat{\mu}_i = 0, \ \hat{\mu}_i \hat{x}_i &= 0, \quad \mathrm{for} \ i \neq r; \label{TT1-kkt1}\\
\hat{x}_i - y_i - \hat{\mu}_i  = 0, \ \hat{\mu}_i \hat{\bm x}^\top \bm 1 &= 0, \quad \mathrm{for} \ i = r; \label{TT1-kkt2}\\
\hat{\mu}_1, \ldots, \hat{\mu}_p &\geq 0; \label{TT1-kkt3}
\end{align}
Therefore, for any $i \neq r$, it holds that $\hat{x}_i = y_i + \hat{\mu}_r - \hat{\mu}_i$, $\hat{\mu}_i \hat{x}_i = 0$, $\hat{\mu}_i \geq 0$ and $\hat{x}_i \leq 0$. As a result, we obtain the following results:
\begin{itemize}
\item If $y_i + \hat{\mu}_r \leq 0$, then $\hat{\mu}_i = 0$, indicating that $\hat{x}_i = y_i + \hat{\mu}_i$.

\item If $y_i + \hat{\mu}_r > 0$, then $\hat{\mu}_i = y_i + \hat{\mu}_r$, indicating that $\hat{x}_i = 0$.
\end{itemize}
Overall, we obtain that
\begin{equation}\label{eq:solution_xi}
\hat{x}_i = \min (y_i + \hat{\mu}_r, \, 0), \quad \forall \, i \neq r.
\end{equation}

On the other hand, $\hat{x}_r$ and $\hat{\mu}_r$ satisfy that $\hat{x}_r = y_r + \hat{\mu}_r$, $\hat{\mu}_r \hat{\bm x}^\top \bm 1 = 0$, $\hat{\mu}_r \geq 0$, and $\hat{\bm x}^\top \bm 1 \geq 0$. Then we have the following results:
\begin{itemize}
\item If $y_r \geq -\sum_{i \in [p] \setminus r} \min (y_i, \, 0)$, then $\hat{\mu}_r = 0$, indicating that $\hat{x}_r = y_r$. Following from \eqref{eq:solution_xi}, $\hat{x}_i = \min (y_i, \, 0)$ for any $i \neq r$.

\item If $y_r < -\sum_{i \in [p] \setminus r} \min (y_i, \, 0)$, then $\hat{\mu}_r \neq 0$. This is because $\hat{\mu}_r = 0$ will result in $\hat{\bm x}^\top \bm 1 <0$, which does not fulfill the KKT system. Together with the KKT condition that $\hat{\mu}_r \hat{\bm x}^\top \bm 1 = 0$, one has $\hat{\bm x}^\top \bm 1 = 0$. Therefore, one obtains that $\hat{x}_r = y_r + \hat{\mu}_r$, where $\hat{\mu}_r$ satisfies $\hat{\mu}_r + y_r + \sum_{i \in [p] \setminus r} \min (y_i + \hat{\mu}_r, \, 0) =0$.
\end{itemize}
Then the $\rho$ in Proposition \ref{solution-PB} is the dual optimal point $\hat{\mu}_r$ above.
\end{proof}

\subsection{Proofs of Theorem~\ref{theorem 2} and Corollary~\ref{corollary-result}}\label{proof-sec}
\begin{proof}

We prove the theorem under the case of estimating \textit{M}-matrices, \textit{i.e.}, $\mathcal{S} = \mathcal{M}^p$. The proof can be directly applied to the case of estimating diagonally dominant \textit{M}-matrices, \textit{i.e.}, $\mathcal{S} = \mathcal{M}_D^p$, since $\mathcal{M}_D^p$ is a subset of $\mathcal{M}^p$. Provided that the following event holds:
\begin{equation}\label{T1-event-J}
\mathcal{J} := \Big\lbrace \big\| \bm \Sigma^\star  - \widehat{\bm \Sigma} \big\|_{\max} \leq \frac{\sqrt{2}}{2}\lambda \Big\rbrace.
\end{equation}

Recall that $\widehat{\bm \Theta}^{(k)} = G_{\lambda, \widehat{\bm \Sigma}} \big( \widehat{\bm \Theta}^{(k-1)} \big)$ for $k \geq 2$. According to Assumption \ref{assumption 1} that $p_\lambda (0) =\lambda$, we can rewrite $\widehat{\bm \Theta}^{(1)}$ as $\widehat{\bm \Theta}^{(1)} = G_{\lambda, \widehat{\bm \Sigma}} \big( \widehat{ \bm \Theta}^{(0)} \big)$, where $\widehat{\Theta}^{(0)}_{ij} = 0$ for any $i \neq j$. Thus, $\big| T_{\alpha} \big (\widehat{\bm \Theta}^{(0)} \big) \big| \leq 2s$. Following from Lemma \ref{lem3}, we can further obtain that $\big|  T_{\alpha} \big (\widehat{\bm \Theta}^{(1)} \big) \big| \leq 2s$. To simplify notation, we denote $T_{\alpha} \big (\widehat{\bm \Theta}^{(k)} \big)$ by $T_{\alpha}^k$ for short. By induction, we obtain that, for any $k \geq 1$, 
\begin{equation}\label{eq:induction_T}
\big|  T_{\alpha} \big (\widehat{\bm \Theta}^{(k)} \big) \big| \leq 2s,
\end{equation}
and
\begin{equation}\label{q7}
\big\| \widehat{\bm \Theta}^{(k)} - \bm \Theta^{\star} \big\|_{\mathrm{F}} \leq \frac{c_0}{2} \big( \delta_1^{k-1} + \delta_2^{k-1} \big),
\end{equation} 
where $\delta_1^{k-1} = \big\| p_{\lambda} \big( \big| \widehat{\bm \Theta}^{(k-1)}_{S^\star} \big| \big) \big\|$ and $\delta_2^{k-1} = \big\| \big(  \bm \Sigma^{\star} - \widehat{\bm \Sigma}  \big)_{T_\alpha^{k-1} \cup I_p} \big\|$.

In what follows, we show that the term $\delta_1^{k-1}$ in \eqref{q7} can be bounded in terms of $\big\| \widehat{\bm \Theta}^{(k-1)} - \bm \Theta^{\star} \big\|$. For any $(i,j) \in S^\star$ and $\bm \Theta \in \mathbb{R}^{p \times p}$, if $\left| \Theta_{ij}^{\star} - \Theta_{ij} \right| \geq \alpha$, then one has
\begin{equation}
0 \leq p_{\lambda}\left( |\Theta_{ij}| \right) \leq \lambda  \leq \frac{\lambda}{\alpha} \left| \Theta_{ij}^{\star} - \Theta_{ij} \right|, \nonumber
\end{equation}
where the first two inequalities follows from Assumption \ref{assumption 1}. If $\left| \Theta_{ij}^{\star} - \Theta_{ij} \right| \leq \alpha$, then one has
\begin{equation}
0 \leq p_{\lambda}\left( |\Theta_{ij}| \right) \leq p_{\lambda } \left( \big|\Theta_{ij}^{\star} \big| - \alpha \right) = 0,
\end{equation}
where the second inequality follows from Assumption \ref{assumption 1} that $p_{\lambda }$ is monotonically non-increasing, and the equality is established because $\min_{(i, j)\in S^\star } \big| \Theta_{ij}^{\star} \big| - \alpha \geq \gamma \lambda$ and $p_{\lambda}(x)=0$ for any $x \geq \gamma \lambda$ following from Assumptions \ref{assumption 1} and \ref{assumption 2}. As a result, one can obtain
\begin{equation}\label{eq:hl}
\delta_1^{k-1} \leq \frac{\lambda}{\alpha} \big\| \big(\widehat{\bm \Theta}^{(k-1)} - \bm \Theta^{\star} \big)_{S^\star} \big\|
\leq \frac{\lambda}{\alpha} \big\| \widehat{\bm \Theta}^{(k-1)} - \bm \Theta^{\star} \big\|_{\mathrm{F}}.
\end{equation}

To bound $\delta_2^{k-1}$, we separate the set $T_{\alpha}^{k-1}$ into two parts, $S^\star$ and $L_{\alpha}^{k-1} \backslash S^\star$, where $L_{\alpha}^k$ denotes the set $L_{\alpha} \big (\widehat{\bm \Theta}^{(k)} \big)$. The term $\big\| \big( \bm \Sigma^{\star}  - \widehat{\bm \Sigma}  \big)_{L_{\alpha}^{k-1} \backslash S^\star}\big\|$ can be bounded in terms of $\big\| \widehat{\bm \Theta}^{(k-1)} - \bm \Theta^{\star}\big\|_{\mathrm{F}}$ as follows:
\begin{equation*}
\big\| \big( \bm \Sigma^{\star}  - \widehat{\bm \Sigma}  \big)_{L_{\alpha}^{k-1} \backslash S^\star}\big\|  \leq \sqrt{ \left| L_{\alpha}^{k-1} \backslash S^\star \right|} \big\| \big( \bm \Sigma^{\star}  - \widehat{\bm \Sigma}  \big)_{L_{\alpha}^{k-1} \backslash S^\star } \big\|_{\max}  \leq \frac{\sqrt{2}\lambda}{2\alpha}\big\| \widehat{\bm \Theta}^{(k-1)} - \bm \Theta^{\star} \big\|_{\mathrm{F}},     
\end{equation*}
where the last second equality follows from
\begin{equation}\label{q5-update}
\sqrt{ \left| L_{\alpha}^{k-1} \backslash S^\star \right|}  \leq  \frac{\big\| \widehat{\bm \Theta}^{(k-1)}_{ L_{\alpha}^{k-1} \backslash S^\star } \big\| }{\alpha}  \leq \frac{\big\| \widehat{\bm \Theta}^{(k-1)} - \bm \Theta^{\star} \big\|_{\mathrm{F}}}{\alpha}, 
\end{equation}
where the first inequality follows from the definition of $L_\alpha^{k-1}$. 
Thus one has
\begin{equation} \label{q9}
\delta_2^{k-1} \leq \big\| \big( \bm \Sigma^{\star} - \widehat{\bm \Sigma} \big)_{S^\star \cup I_p} \big\| +  \frac{\sqrt{2}\lambda}{2\alpha}\big\| \widehat{\bm \Theta}^{(k-1)} - \bm \Theta^{\star} \big\|_{\mathrm{F}}.
\end{equation}
Substituting \eqref{eq:hl} and \eqref{q9} into \eqref{q7} yields
\begin{equation*}
\big\| \widehat{\bm \Theta}^{(k)} - \bm \Theta^{\star} \big\|_{\mathrm{F}}  
\leq \frac{c_0}{2}\big\| \big( \bm \Sigma^{\star} - \widehat{\bm \Sigma}  \big)_{S^\star \cup I_p} \big\| + \rho \big\| \widehat{\bm \Theta}^{(k-1)} - \bm \Theta^{\star} \big\|_{\mathrm{F}},
\end{equation*}
where $\rho = \frac{2+\sqrt{2}}{4}$. By induction, if $d_k \leq a_0 +\rho d_{k-1}$ for any $k \geq 1$ with $\rho \in [0, 1)$, then
\begin{equation}\label{induction}
d_k \leq \frac{1-\rho^{k}}{1-\rho}a_0 + \rho^{k}d_0.         
\end{equation}
Applying \eqref{induction} with $a_0 = \frac{c_0}{2}\big\| \big( \bm \Sigma^{\star}  - \widehat{\bm \Sigma}  \big)_{S^\star \cup I_p} \big\|$, and $d_k = \big\| \widehat{\bm \Theta}^{(k)} - \bm \Theta^{\star} \big\|_{\mathrm{F}}$ yields
\begin{equation*}
\big\| \widehat{\bm \Theta}^{(k)} - \bm \Theta^{\star} \big\|_{\mathrm{F}} \leq 4 c_0 \big\| \big( \bm \Sigma^{\star} - \widehat{\bm \Sigma}  \big)_{S^\star \cup I_p} \big\| + \rho^{k} \big\| \widehat{\bm \Theta}^{(0)} - \bm \Theta^{\star} \big\|_{\mathrm{F}}.       
\end{equation*}

Next, we establish Corollary~\ref{corollary-result}. Under the event $\mathcal{J}$, one has
\begin{equation}\label{eq:bound-theta-k}
\big\| \big( \bm \Sigma^{\star}  - \widehat{\bm \Sigma}  \big)_{S^\star \cup I_p} \big\| \leq \frac{\sqrt{2}}{2} \lambda\sqrt{s},
\end{equation}
where the inequality follows from $| S^\star \cup I_p|\leq s$. As a result, we plug $\lambda = \sqrt{ c_1 \tau \log p /n}$ into \eqref{eq:bound-theta-k} and obtain
\begin{equation}
\big\| \widehat{\bm \Theta}^{(k)} - \bm \Theta^{\star} \big\|_{\mathrm{F}} \leq c \sqrt{s \log p/n} + \rho^{k} \big\| \widehat{\bm \Theta}^{(0)} - \bm \Theta^{\star} \big\|_{\mathrm{F}},       
\end{equation}
where $c = 2 c_0 \sqrt{2 c_1 \tau}$.

Finally, we calculate the probability that event $\mathcal{J}$ holds. We apply \citep[Lemma~1]{ravikumar2011high} and union sum bound, and obtain
\begin{equation*}
\mathbb{P} \Big[ \big\| \left( \bm \Theta^\star  \right)^{-1} - \widehat{\bm \Sigma} \big\|_{\max} \geq \frac{\sqrt{2}}{2}\lambda \Big]\leq 4 p^2 \exp \left( - \frac{n \lambda^2}{ \big(80 \max_i \Sigma_{ii}^\star \big)^2} \right).
\end{equation*}
Take $\lambda= \sqrt{ c_1 \tau \log p/n}$ with $\tau >2$ and $c_1 = \big( 80 \max_i \Sigma_{ii}^\star \big)^2$. By calculation, we obtain that
\begin{equation}\label{prob_bound}
\mathbb{P} \Big[ \big\| \left( \bm \Theta^\star  \right)^{-1} - \widehat{\bm \Sigma} \big\|_{\max} \leq \frac{\sqrt{2}}{2}\lambda \Big] \geq 1-4/p^{\tau-2},
\end{equation}
completing the proof.
\end{proof}

\end{document}